\documentclass[journal]{IEEEtran}

\ifCLASSINFOpdf
\usepackage[pdftex]{graphicx}
\graphicspath{{../pdf/}{../jpeg/}}
\DeclareGraphicsExtensions{.pdf,.jpeg,.png}
\else
\usepackage[dvips]{graphicx}
\graphicspath{{../eps/}}
\fi
\usepackage{booktabs} % For better table rules
\usepackage{graphicx}
\usepackage{epstopdf}
\usepackage{diagbox}
\usepackage{multirow}
\usepackage[cmex10]{amsmath}
\usepackage[ruled]{algorithm2e} 
\usepackage{url}
\usepackage{lineno,hyperref}
\usepackage{amsthm}%%proof
\usepackage{booktabs}
\usepackage{subfigure}
\usepackage{url}
\usepackage{mathrsfs}%话题字母% $\mathscr{L}$
\usepackage{amsmath}
\usepackage{amsthm}
\usepackage{bbding}%%%true or false

\usepackage{amsfonts,amssymb}%空心字母
\usepackage{mathrsfs}%话题字母% $\mathscr{L}$
\usepackage{amsmath}
\usepackage{amsthm}
\usepackage{bbding}%%%true or false
\usepackage{cite}
\usepackage{latexsym,amsmath,xcolor,multicol,booktabs,calligra}

\usepackage{fancyhdr}
\begin{document}
	\title{Uncertainty Quantification via Hölder Divergence for Multi-View Representation Learning}
	\author{
        Yan Zhang$^{\dagger}$, Ming Li$^{\dagger}$, Chun Li$^{*}$, ~\IEEEmembership{Member,~IEEE}, Zhaoxia Liu, Ye Zhang, and F. Yu,~\IEEEmembership{Fellow,~IEEE}
        \thanks{This research was supported by the National Key Research and Development Program of China (No. 2022YFC3310300), Guangdong Basic and Applied Basic Research Foundation (No. 2024A1515011774), the National Natural Science Foundation of China (No. 12171036), Shenzhen Sci-Tech Fund (Grant No. RCJC20231211090030059), and Beijing Natural Science Foundation (No. Z210001). The source code for our method is publicly available at: \url{https://github.com/wmh12138/HDMVL.}}
		\thanks{Yan Zhang, C. Li and Ye Zhang are with MSU-BIT-SMBU Joint Research Center of Applied Mathematics, Shenzhen MSU-BIT University, Shenzhen, 518172, China.} 
        \thanks{Yan Zhang and Z. Liu also are with School of Science, Minzu University of China, Beijing, 100081, China. E-mail: za1234yuuy@gmail.com; liuzhaoxia@muc.edu.cn.}
        \thanks{Ye Zhang is also with School of Mathematics and Statistics, Beijing Institute of Technology, 100081, Beijing, China.}
        \thanks{M. Li and F. Yu are with Guangdong Laboratory of Artificial Intelligence and Digital Economy (SZ), Shenzhen, 518083, China. E-mail: ming.li@u.nus.edu.}
		\thanks{$^{\dagger}$Co-first authors. *Corresponding author: Chun Li (E-mail: lichun2020@smbu.edu.cn).}
	}
	\markboth{IEEE Transactions on Multimedia}%
	{Shell \MakeLowercase{\textit{et al.}}: Bare Demo of IEEEtran.cls for IEEE Journals}
	\maketitle

	\begin{abstract}
        Evidence-based deep learning represents a burgeoning paradigm for uncertainty estimation, offering reliable predictions with negligible extra computational overheads. Existing methods usually adopt Kullback-Leibler divergence to estimate the uncertainty of network predictions, ignoring domain gaps among various modalities. To tackle this issue, this paper introduces a novel algorithm based on Hölder Divergence (HD) to enhance the reliability of multi-view learning by addressing inherent uncertainty challenges from incomplete or noisy data. Generally, our method extracts the representations of multiple modalities through parallel network branches, and then employs HD to estimate the prediction uncertainties. Through the Dempster-Shafer theory, integration of uncertainty from different modalities, thereby generating a comprehensive result that considers all available representations. Mathematically, HD proves to better measure the ``distance'' between real data distribution and predictive distribution of the model and improve the performances of multi-class recognition tasks. 
        Specifically, our method surpasses the existing state-of-the-art counterparts on all evaluating benchmarks.
        We further conduct extensive experiments on different backbones to verify our superior robustness. It is demonstrated that our method successfully pushes the corresponding performance boundaries. Finally, we perform experiments on more challenging scenarios, \textit{i.e.}, learning with incomplete or noisy data, revealing that our method exhibits a high tolerance to such corrupted data. 
	\end{abstract}
	
	
	\IEEEpeerreviewmaketitle

\section{Introduction}
\label{sec:intro}
\begin{figure}%[t]
	\centering
	\includegraphics[width=\linewidth]{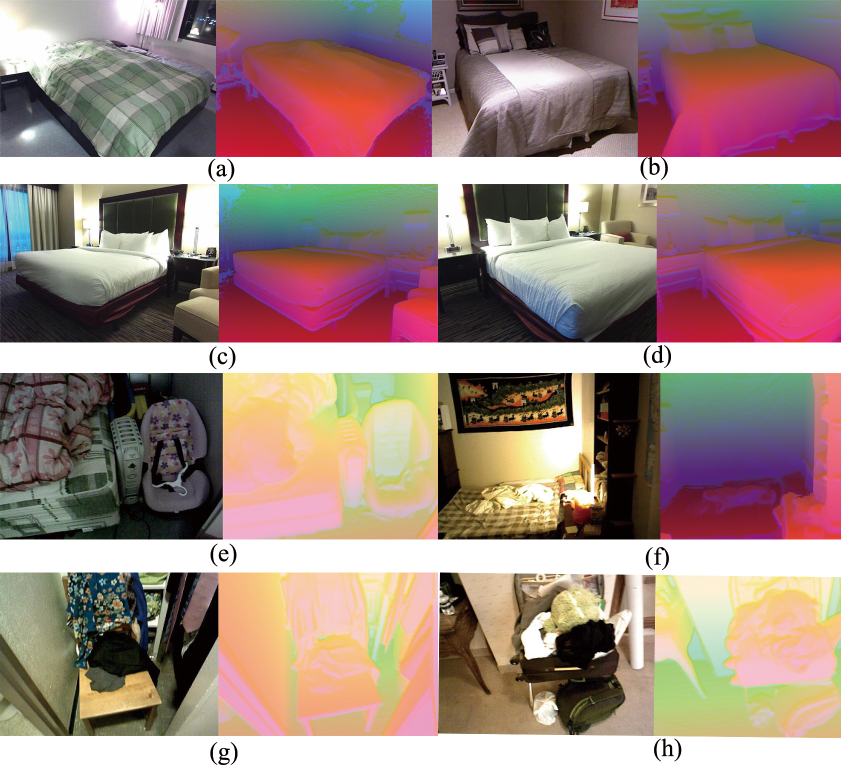}
	\caption{The confident (a-d) and uncertain (e-h) sample-depth pairs predicted by our method on the SUNRGBD \cite{pp46} dataset. The comparison reveals the discrepancies between high-confident and uncertain predictions, demonstrating the capacity of our method in handling challenging cases.}
	\label{fig_3_add}
\end{figure}
Recently, multi-view learning has become pivotal in machine learning, addressing diverse forms of multi-view data \cite{pp59,ppp75,zz67,zz68}.
In the field of multi-view learning, researchers have found that the performance of models can be improved by estimating the uncertainty of data distribution. However, incorporating uncertainty considerations in each modality for reliable predictions remains a gap. 

There are two categories in the estimation methods of uncertainty.
The first category often assigns equal weights to modalities, lacking practicality \cite{pp5}. The second dynamically assigns weights to each modality, considering uncertainty to avoid unreliable predictions \cite{pp6}. And Confident (a-d) and uncertain (e-h) samples from the SUN RGB-D test set is shown in Fig. \ref{fig_3_add}. Regardless of the approach, estimating the uncertainty in the classification results, especially the distribution uncertainty, is critical to the reliability of the model. Current methods often use the Kullback-Leibler divergence (KLD) \cite{pp4} to estimate the uncertainty of the classification results, but challenges persist in accurately discerning distribution uncertainty \cite{pp70}. To address this, we use HD \cite{pp3}, superior in clustering experiments, replaces KLD in the models for more precise classification outcomes.

Our method outperforms existing methods, offering a systematic analysis, identification of critical determinants, and empirical validation across four multi-view scenario datasets. In addition, we also test the performance of the attention mechanism in the field of multimodal image classification. Specifically, this study uses the attention mechanism to extract image features of different modalities, and uses the Visual Transformer (ViT) model \cite{pp72} and the Mamba model \cite{pp71,zz73} to explore the application of different types of attention mechanisms in the field of image recognition.

In summary, the contributions of this paper can be encapsulated as follows:
\begin{itemize}
    \item \textbf{Enhanced Objective Function:} Through an exploration of Hölder divergence's mathematical properties, we elevate the ETMC model's objective function, resulting in the creation of the HDMVL model. Experimental results across four multi-view scenario datasets conclusively demonstrate that the HDMVL model outperforms the original ETMC model in terms of classification accuracy. 
    
    \item \textbf{Divergence Formulas:} We have delved into the impact of utilizing diverse divergence formulas to formulate objective functions concerning classification outcomes. This exploration yields fresh insights into the enhancement of multi-view classification and clustering models, affirming that an improved objective function can significantly boost classification and clustering efficacy. Furthermore, it underscores the favorable influence of Hölder divergence on classification and clustering accuracy and model performance within the realm of multi-view learning tasks. 
    
    \item \textbf{Empirical Validation:} Our extensive empirical experiments provide concrete evidence that Hölder divergence excels over KLD in the context of multi-class classification and clustering tasks, emphasizing its superior performance. It also highlights the adaptability of Hölder divergence to a variety of multi-class classification and clustering tasks, offering the potential for reduced computational costs through adjustments in the Hölder index. Additionally, the experiment proves that the global attention mechanism can better integrate information between different modalities and improve the performance of multimodal classification models.
\end{itemize}
 And the main notations used in this work is shown in Table \ref{tab00}.

\section{Related Works}
\renewcommand\arraystretch{1.0}
\begin{table}[t]
	\setlength{\belowdisplayskip}{0pt}
	\setlength{\abovedisplayskip}{0pt}
	\setlength{\abovecaptionskip}{0pt}
	\centering
	%\begin{center}
	\scriptsize
	\caption{Main Notations Used in This Work. This table outlines the key symbols and notations used throughout the paper, providing clear definitions and units to ensure consistency and understanding of the variables involved.}%\vspace{-0.25cm}
	\setlength{\tabcolsep}{2pt}
	\begin{tabular}{p{3cm}p{5cm}}  %\toprule[1px]
		\toprule [1.0pt]
		% after \\: \hline or \cline{col1-col2} \cline{col3-col4} ...
		Notation&Definition\\
		\midrule[0.5pt]
  	${D_{KL}}(.||.)$& KL divergence \\
		$\alpha, \beta$ &  The conjugate exponents of Hölder \\
		$D_\alpha^H(p(x):q(x))$ & The Hölder pseudo-divergence of $p(x)$ and $q(x)$\\
		$b_k^i$ & Reliability of the $kth$ classification result for the $ith$ modality \\
		${{\rm M}^i} = \left\{ {\{ b_k^i\} _{k = 1}^K,{u^i}} \right\}$&Reliability of the classification result for the ith modality and overall uncertainty\\
		$\left\{ {x_n^m} \right\}_{m = 1}^M,{y_n}$& The $n$ samples with $M$ modalities each, and the labels corresponding to the $n$ samples, respectively\\
		${\lambda _t}, Dir(.|.)$ & Weight parameter and Dirichlet distribution, respectively\\
		\bottomrule[1.0pt]
	\end{tabular}
	%\end{center}
	\label{tab00}%\vspace{-0.45cm}
\end{table}

\textbf{A. Multi-View Learning}
Multi-view learning leverages diverse data perspectives to enhance machine learning, improving tasks like classification, clustering, and regression \cite{zz61,zz62,zz63}. Canonical Correlation Analysis (CCA) is a key method, optimizing linear feature combinations across views to maximize correlation \cite{pp10}. Recently, contrastive learning and deep multi-view learning, driven by neural networks, have further advanced this field by improving performance and model sophistication \cite{pp14}. Moreover, Wu et al. \cite{ppp76} proposed a Self-Weighted Contrastive Fusion method for deep multi-view clustering, which enhances clustering performance by learning a balanced fusion of multiple views while preserving the most informative features from each view. Tan et al. \cite{ppp77} present a method for unsupervised multi-view clustering that integrates and refines knowledge from both individual views and cross-view interactions to improve clustering performance. Gou et al. \cite{ppp78} proposes Reconstructed Graph Constrained Auto-Encoders, a novel framework for improving multi-view representation learning by incorporating graph structure constraints into the auto-encoder architecture.

%\noindent\textbf{Multi-View Learning.} Multi-view learning revolutionizes machine learning by capitalizing on diverse data perspectives, extending beyond single-type image recognition. It extracts valuable insights, leading to advancements in classification, clustering, and regression tasks. Canonical Correlation Analysis (CCA) is a classical method extensively utilized in multi-view learning. It identifies optimal linear combinations of features within each view, maximizing correlation \cite{pp10}. Complementing CCA, contrastive learning techniques have recently gained prominence, exhibiting the potential to deliver substantial performance improvements. Recently, the emergence of deep multi-view learning, underpinned by neural networks, streamlines feature extraction, enabling the development of intricate models that surpass traditional methods \cite{pp14}. 

\textbf{B. Evidence Theory} 
Dempster-Shafer theory \cite{pp7}, introduced by Glenn Shafer in 1976, is a mathematical framework for managing uncertainty and inference \cite{zz64,zz65}. Key principles include evidence, basic belief assignment, combination, and belief functions. Widely applied in machine learning, data mining, and medical diagnosis, it offers robust tools for handling large datasets and uncertainty. In multi-view learning, it enhances information integration from multiple sources, particularly through improved Dempster's combination rule \cite{pp32}. For instance, Li et al. \cite{pp32} improved multispectral pedestrian detection using confidence-aware fusion based on Dempster-Shafer theory. Zhang et al. \cite{ppp79} proposed a novel data augmentation method that combines Mixup and Dempster-Shafer theory to enhance model robustness and uncertainty estimation in machine learning tasks. Li et al. \cite{ppp80} proposed a confidence-aware fusion method based on Dempster-Shafer theory to enhance the accuracy and reliability of multispectral pedestrian detection.

%\noindent\textbf{Evidence Theory.} The Dempster-Shafer theory \cite{pp7}, introduced by Glenn Shafer in 1976, is a mathematical framework for uncertainty handling and inference. The theory's key principles include evidence, basic belief assignment, combination, and belief function. It is widely used in machine learning, data mining, and medical diagnosis, providing robust tools for large dataset management and uncertainty handling. In multi-view learning, it enhances information integration from multiple sources, particularly improving Dempster's combination rule \cite{pp32}. For example, Li et al. \cite{pp32} introduced a novel approach for improving multispectral pedestrian detection through confidence-aware fusion using Dempster-Shafer theory.

\textbf{C. Uncertainty Estimation} Despite the success of deep learning \cite{yue2021rnn, tian2023fakepoi, li2023dr, zhang2024semi, lyu2021treernn, li2024instant3d, li2021exploiting, li2021self, zheng2020lodonet, li2023stprivacy}, managing uncertainty remains a significant challenge \cite{zz66}. Uncertainty arises from incomplete or noisy data and complicates decision-making processes in real-world scenarios. Deep neural networks struggle with both data and model uncertainty, as well as accurately propagating uncertainty from inputs to outputs. Robust solutions are needed to address these issues. Recent advances in deep learning for uncertainty estimation include Bayesian methods, uncertainty quantification, and automated machine learning. Bayesian neural networks, which combine deep learning with Bayesian statistics, have been a focus since the 1990s. Monte Carlo methods, such as Monte Carlo Dropout, are also valuable for uncertainty estimation. Recent work on the Dirichlet distribution has further advanced the field. For example, Han et al. \cite{pp6} introduced the Enhanced Trusted Multi-View Classification (ETMC) algorithm to improve multi-view classification.

%\noindent\textbf{Uncertainty Estimation.} Deep learning, a cornerstone of artificial intelligence, has excelled in image classification, natural language processing, and autonomous driving. However, uncertainty remains a challenging aspect despite its success. Uncertainty, stemming from incomplete or noisy information, is prevalent in real-world scenarios, posing challenges for decision-making processes. Deep neural networks grapple with data and model uncertainty, along with propagation issues. The propagation of uncertainty from inputs to outputs lacks accurate estimation, necessitating robust solutions. Recent trends in deep learning uncertainty include Bayesian methods, uncertainty quantification, and automated machine learning. Bayesian neural networks, dating back to the 1990s, combine deep learning and Bayesian statistics. Monte Carlo methods \cite{pp22}, like Monte Carlo Dropout, offer valuable techniques for uncertainty estimation. Recent proposals rooted in the Dirichlet distribution highlight advancements. Han et al. \cite{pp6} proposed the enhanced trusted multi-view classification (ETMC) algorithm enriches multi-view classification. %Li et al. \cite{pp75} proposed a prototype-based method for quantifying aleatoric uncertainty in cross-modal retrieval tasks.

\newtheorem{definition}{\bf{Definition}}
\begin{definition} 
	\label{def_3}
	(\textbf{Hölder Statistical Pseudo-Divergence, HPD \cite{pp3}}) HPD pertains to the conjugate exponents $\alpha$ and $\beta$, where $\alpha \beta>0$. In the context of two densities, $p(x) \in {L^\alpha }\left( {\Omega,\nu } \right)$ and $q(x) \in {L^\beta }\left( {\Omega ,\nu } \right)$, both of which belong to positive measures absolutely continuous with respect to $\nu$, HPD is defined as the logarithmic ratio gap, as follows: $D_{\alpha}^{H}(p(x):q(x))=-\log\left(\frac{\int_{\Omega}p(x)q(x)\mathrm{d}x}{\left(\int_{\Omega}p(x)^{\alpha}\mathrm{d}x\right)^{\frac1\alpha}\left(\int_{\Omega}q(x)^{\beta}\mathrm{d}x\right)^{\frac1\beta}}\right).$ When $0<\alpha<1$ and $\beta  = \bar \alpha  = \frac{\alpha }{{\alpha  - 1}} < 0$ or $\alpha<0$ and $0<\beta<1$, the reverse HPD is defined by: $\begin{aligned}D_\alpha^\mathrm{H}(p(x):q(x))&=\log\left(\frac{\int_{\Omega}p(x)q(x)\mathrm{d}x}{\left(\int_{{\Omega}}p(x)^\alpha\mathrm{d}x\right)^{\frac1\alpha}\left(\int_{{\Omega}}q(x)^\beta\mathrm{d}x\right)^{\frac1\beta}}\right).\end{aligned}$
\end{definition}

\begin{definition} 
	\label{def_1}
	(\textbf{Dirichlet Distribution \cite{ppp57}}) The Dirichlet distribution of order $K$ (where $K\ge 2$) with parameters $\alpha_i>0, i=1,2,3...,K$ is defined by a probability density function with respect to Lebesgue measure on the Euclidean space $R^{K-1}$ as follows: ${\rm{Dirichle}}{{\rm{t}}_n}({\mu _1}, \cdots ,{\mu _K}|{\alpha _1}, \ldots ,{\alpha _K}) = \frac{{\Gamma \left( {\sum\limits_{i = 1}^n {{\alpha _i}} } \right)}}{{\prod\limits_{i = 1}^n \Gamma  ({\alpha _i})}}\prod\limits_{i = 1}^n {\mu _i^{{\alpha _i} - 1}},$ where ${\mu_i} \in {S_K}$, and ${S_K}$ is the standard $K-1$ dimensional simplex, namely, $${{\cal S}_K} = \left\{ {\left( {{\mu _1},{\mu _2},...,{\mu _K}} \right)\mid \sum\limits_{i = 1}^K {{\mu _i}}  = 1,\;0 \le {\mu _1}, \ldots ,{\mu _K} \le 1} \right\},$$ and $\Gamma(.)$ is the gamma function,  defined as: $\Gamma(s)=\int_0^\infty x^{s-1}\mathrm{e}^{-x}\mathrm{~d}x,\quad s>0$.
\end{definition}
\begin{definition} 
	\label{def_5}
	(\textbf{Exponential Family Distribution \cite{p13}}) The probability density function of the Dirichlet distribution is expressed as follows: $p\left( {x;\theta } \right) = \exp \{ {\theta ^ \top }T(x) - F(\theta ) + B(x)\},$ where $\theta$ is the natural parameter, $T(x)$ is the sufficient statistic, $F(\theta)$ is the log-normalizer, and $B(x)$ is the base measure.
\end{definition}
\begin{definition} 
	\label{def_6}
	(\textbf{The Exponential form of the Dirichlet Distribution \cite{pp58}}) Exponential formulation of the Dirichlet distribution probability density function can be rewrite as follows: 
	\begin{equation}
		\label{equ_12}
		\begin{array}{l}
			{\rm{Dirichle}}{{\rm{t}}_n}({\mu _1}, \cdots ,{\mu _K}|{\alpha _1}, \ldots ,{\alpha _K})\\
			\begin{array}{*{20}{c}}
				{}
			\end{array} = \exp \left\{ {\sum\limits_i^K {\left( {{\alpha _i} - 1} \right)} \log {\mu _{\rm{i}}} - \left[ {\begin{array}{*{20}{c}}
						{\sum\limits_i^K {\log \Gamma \left( {{\alpha _i}} \right)} }\\
						{ - \log \Gamma \left( {\sum\limits_i^K {{\alpha _i}} } \right)}
				\end{array}} \right]} \right\},
		\end{array}
	\end{equation}
\end{definition}
Allowing us to obtain the canonical form terms: ${\nabla _\theta }T(\theta ) = \left[ {\begin{array}{*{20}{c}}
		{\psi ({\alpha _1}) - \psi (\sum\limits_{i = 1}^K {{\alpha _i}} )}\\
		\vdots \\
		{\psi ({\alpha _K}) - \psi (\sum\limits_{i = 1}^K {{\alpha _i}} )}
\end{array}} \right],$ $\theta=\boldsymbol{\alpha}$, $T(\boldsymbol{\mu})=ln(\boldsymbol{\mu})$, $F(\eta ) = \sum\limits_{i = 1}^K {\ln } \Gamma ({\alpha _i}) - \ln \Gamma (\sum\limits_{i = 1}^K {{\alpha _i}} )$, $B(\boldsymbol{\mu})=-ln(\boldsymbol{\mu})$, and $\psi$ is the digamma function, defined as: $\psi(x)=\frac{\mathrm{d}}{\mathrm{d}x}\ln\Gamma(x)$.

\section{Methodology}	
\subsection{Exploring Multi-Class Classification with Variational Dirichlet Modeling} 

In the field of machine learning, where the representation of compositional data is an integral part of addressing multi-class classification problems, Aitchison \cite{pp34} introduced the Dirichlet distribution as the primary candidate for modeling such data. Mathematically, within a multi-class classification problem involving $K$ classes, the aim is to determine a function to generate a predicted class label, with the overarching objective of minimizing the disparity between this predicted class label and the ground truth. Generally speaking, in deep learning, it is customary to employ the softmax operator to transform the continuous model output into a set of class probabilities. However, it is worth noting that the softmax operator often leads to overconfidence \cite{pp6}.	

The Dirichlet distribution, indeed, stands as a versatile and pivotal probability distribution, particularly when it comes to modeling multi-classification problem and Bayesian inference. Its status as the conjugate prior for the multinomial distribution lends it immense utility in Bayesian statistics, as it ensures that the posterior distribution maintains the same form as the prior \cite{pp38}. This property greatly simplifies the process of Bayesian inference and renders it analytically tractable.

The Dirichlet distribution is a versatile tool in probabilistic modeling, offering flexibility, interpretability, and computational advantages, making it suitable for various applications such as Bayesian statistics, natural language processing, and machine learning. Its key advantages include flexibility in modeling categorical data, conjugacy with the multinomial distribution for Bayesian inference, parameter interpretability, smoothing capabilities, and suitability for generative and hierarchical modeling tasks \cite{pp38}.
\begin{figure*}[t]
	\centering
	\includegraphics[width=0.70\linewidth]{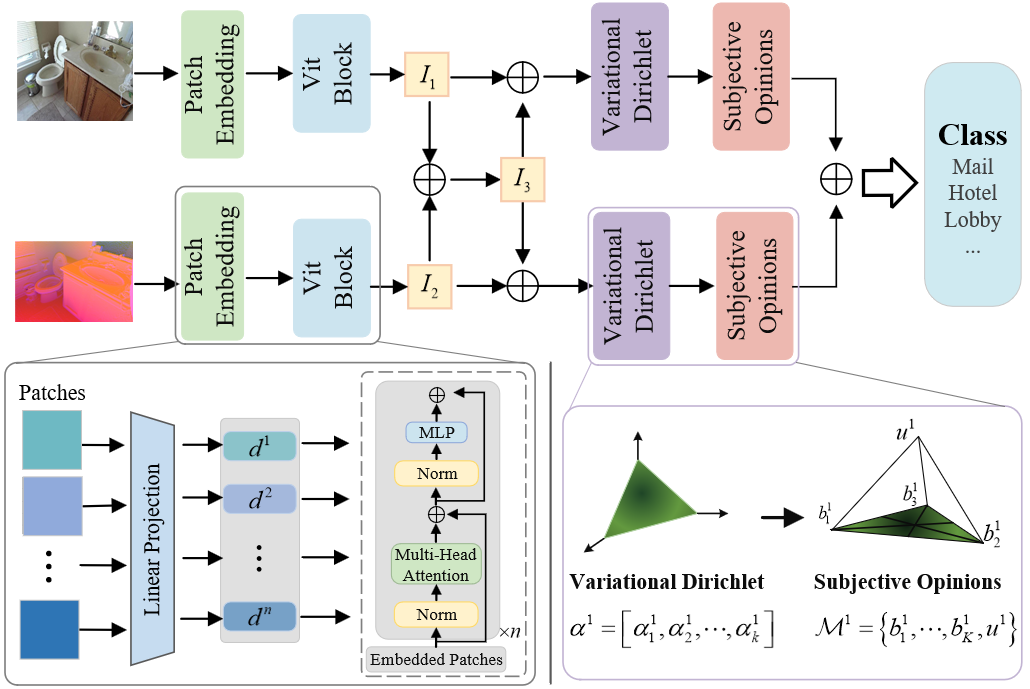}
	\caption{Overview of Uncertainty Estimation via Hölder Divergence for Multi-View Representation Learning. The image features from different modalities are extracted and classified by three separately trained networks. Then, the reliability (${b_k}$) and uncertainty ($\boldsymbol{\mu}$) of the classification results are estimated using Hölder Divergence (HD). Finally, modal fusion is performed based on the reliability and uncertainty sets ${{\rm M}^i}$, where $i$ represents the modality index. This figure illustrates the process of uncertainty quantification and fusion in multi-view learning.}
	\label{fig_1}
\end{figure*}
In multi-view classification, Dirichlet learning offers unique advantages by modeling dependencies between different data views through a stochastic process. It can handle variable-dimensional feature spaces and incorporate prior knowledge effectively, enhancing classification performance and interpretability \cite{pp38}.

For instance, the class probabilities, represented as $\boldsymbol{\mu}=[\mu_1,\cdots,\mu_K]$, can be interpreted as parameters within a multinomial distribution, where $\sum\limits_{k = 1}^K {{\mu _k}}  = 1$. This distribution characterizes the likelihood of $K$ mutually exclusive events occurring \cite{pp39}. On the other hand, the Dirichlet distribution can be employed to capture uncertainty and mitigate issues of overconfidence. Given these considerations, our primary goal is to derive a Dirichlet distribution, which serves as the conjugate prior for the multinomial distribution, thereby allowing us to establish a predictive distribution. Since the consideration of the Dirichlet distribution, we commence by presenting the definition of the exponential family, given its association with this distribution. %The density function of a Dirichlet distribution, can be expressed in Definition \ref{def_1}.

\subsection{Multi-View Classification with Uncertainty-Aware Variational Dirichlet Learning}
“Multi-view classification with uncertainty-aware variational Dirichlet learning” is an enhanced algorithm based on the trusted multi-view classification (TMC) algorithm. In trusted multi-view classification, the process involves the acquisition of class probabilities from different modalities, followed by the modeling of these class probabilities using a Dirichlet distribution to derive the distribution of classification results. This process yields “evidence” regarding the reliability of the classification. Subsequently, utilizing this evidence and employing evidence theory, the algorithm computes the confidence and uncertainty associated with the classification results. Finally, the Dempster-Shafer theory, a method for probabilistic reasoning, is utilized to fuse the classification results obtained from various modalities. However, within the TMC algorithm, the interaction between different modalities occurs primarily at the decision-making level, which can potentially limit its performance in specific scenarios.

To illustrate, let's consider a smart home system employing the TMC algorithm, which is divided into three views: data collection, processing, and control. If interactions between these views are limited to the control layer, a situation might arise where a user wishes to adjust room temperature using a smartphone application. The absence of a direct mechanism to link the data collection and data processing views can result in delays or operational errors.

In response to this challenge, researchers introduced the enhanced trusted multi-view classification algorithm. This enhancement involves the introduction of an additional “pseudo-view” to facilitate interactions between different perspectives. The pseudo-view is generated based on the original model and shares similar structural elements and parameters. It serves as an extension or complement to the original model, enabling the inclusion of additional viewpoints or information sources. By incorporating the pseudo-view, new perspectives can be seamlessly integrated into the existing model, enhancing performance through the utilization of multiple viewpoints and information sources. For instance, in natural language processing tasks, the primary view could be a statistically trained language model, while a neural network-based semantic representation is introduced as a pseudo-view. This enables the system to achieve a more comprehensive understanding of textual content, thus enhancing its expressiveness and inferential capabilities. Empirical results demonstrate that the ETMC algorithm outperforms the TMC algorithm on multi-view datasets. Consequently, in our research, we adopt the ETMC algorithm to achieve our objectives.

\subsection{Uncertainty Analysis}
In the ETMC algorithm, modality fusion is primarily grounded in subjective logic \cite{pp40} and Dempster-Shafer's theory \cite{pp7}. Throughout the training process, it is imperative to conduct a quantitative analysis of the uncertainty and credibility associate with each modality, yielding specific values. Subsequently, a simplified evidence theory is employed to facilitate modal fusion. Furthermore, an assessment of uncertainty and credibility using subjective logic is conducted on the classification results of the fused modalities.

To calculate the uncertainty and credibility of individual modalities in the algorithm, a Dirichlet distribution is introduced. This distribution serves as a “distribution” for the features extracted by the model's classification layer. Confidence in the classification results and the quantification of uncertainty are computed through the Dirichlet distribution and subjective logic. Based on this data, modalities are selectively fused using evidence theory. Additionally, to obtain a Dirichlet distribution, the algorithm replaces the commonly used softmax layer with a non-negative activation function layer. The specific steps are as follows: for a $K$-classification task, each sample contains data from $V$ modalities. For modality ${{\rm M}^1} = \left\{ {\{ b_k^1\} _{k = 1}^K,{u^1}} \right\}$, the uncertainty of confidence in the corresponding classification result can be calculated using the Dirichlet distribution. For ${{\rm M}^2} = \left\{ {\{ b_k^2\} _{k = 1}^K,{u^2}} \right\}$, then employ the simplified evidence theory to calculate the fusion of modality ${\rm M} = {{\rm M}^1} \oplus {{\rm M}^2}$. The simplified fusion rules are given by ${b_k} = \frac{1}{{1 - C}}(b_k^1b_k^2 + b_k^1{u^2} + b_k^2{u^1})$, $u = \frac{1}{{1 - C}}{u^1}{u^2}$. In this scenario, each sample contains data from $V $ modalities, resulting in ${\rm{M}} = {{\rm{M}}^1} \oplus {{\rm{M}}^2} \oplus  \cdots  \oplus {M^V}$.

\subsection{Variational Inference for Hölder Divergence}

A generative model can be expressed as $ p_\theta(x, z) = p_\theta(x|z) p(z) $, where $p_\theta(x|z) $ is the likelihood, and $ p(z) $ is the prior. From the perspective of a Variational Autoencoder (VAE) \cite{pp76}, the true posterior $ p(z|x) $ can be approximated by $ q_\phi(z|x) $. The evidence lower bound (ELBO) $\mathcal{L}_{ELBO}(\theta, \phi; x) $ for VAE can be formulated as:
\begin{equation}
\label{equ_100}
    \mathbb{E}_{q_\phi(z|x)} [\log p_\theta(x|z)] - D_{KL}(q_\phi(z|x) \| p(z)),
\end{equation}

According to the Cauchy–Schwarz regularized autoencoder \cite{pp73}, the objective function incorporating Hölder Statistical Pseudo-Divergence regularization can be specified as $\mathcal{L}_{HLBO}(\theta, \phi; x)$:
\begin{equation}
\label{equ_101}
    \ \mathbb{E}_{q_\phi(z|x)} [\log p_\theta(x|z)] - \lambda D_{\alpha}^{H}(q_\phi(z|x) \| p(z)),
\end{equation}
where $D_{\alpha}^{H} $ denotes the HPD, and $\lambda $ is the regularization parameter. In summary, we derive the overall loss function as follows:
\begin{equation}
	\label{eqn_34}
	\begin{array}{*{20}{l}}
		{{L^{overall}} = \sum\limits_{i = 1}^N {{L^{fused}}} \left( {\left\{ {x_n^m} \right\}_{m = 1}^M,{y_n}} \right)\begin{array}{*{20}{c}}
				{}&{}
		\end{array}}\\
		{\begin{array}{*{20}{c}}
				{}&{}&{}
			\end{array} + \sum\limits_{i = 1}^N {{L^{pseudo}}} \left( {\left\{ {x_n^m} \right\}_{m = 1}^M,{y_n}} \right)}\\
		{\begin{array}{*{20}{c}}
				{}&{}&{}
			\end{array} + \sum\limits_{i = 1}^N {\sum\limits_{m = 1}^M {{L^m}} } \left( {x_n^m,{y_n}} \right).}
	\end{array}
\end{equation}

Now, let's delve into the specific components of the loss function.
\begin{algorithm}[t]
	\caption{\small Uncertainty Estimation via Hölder Divergence for Multi-View Representation Learning.}
	% \caption{}
	\label{alg:spl}
	\DontPrintSemicolon
	\small
	\tcp*[f]{\textbf{*Training*}}\\
	\textbf{Input:} Multi-View Dataset: $D = \left\{ {\left\{ {{\rm X}_n^m} \right\}_{m = 1}^M,{y_n}} \right\}_{n = 1}^N$;\\
	\textbf{initialization:} Initialize the parameters of the neural network.\\
	\While{not converged}  
	{ 
		\For {$m=1:M$}{
			(1) $Dir({\mu ^m}|{x^m}) \leftarrow$ variational network output;\\ %\ref{eqn_40}
			(2) Subjective opinion ${M^m} \leftarrow Dir({\mu ^m}|{x^m});$}\
		
		(1) Obtain joint opinion ${M^m}$;\\
		(2) Obtain $Dir({\mu ^m}|{x^m})$;\\
		(3) Obtain the overall loss by updating $\alpha$ and $\left\{ {{\alpha ^v}} \right\}_{v = 1}^V$;\\
		(4) Maximize \textbf{objective function} and update the networks with gradient descent;\\
	}
	{\bfseries Output:} networks parameters.\\
	\tcp*[f]{\textbf{*Test*}}\\
	Calculate the joint belief and the uncertainty masses.\\
\end{algorithm}
The first term of loss function:
\begin{equation}
	\label{eqn_31}
	\begin{array}{l}
		{L^{fused}}\left( {\left\{ {x_n^m} \right\}_{m = 1}^M,{y_n}} \right)\\
		\begin{array}{*{20}{c}}
			{}&{}
		\end{array} = \left( {\begin{array}{*{20}{c}}
				{{{\mathbb{E}}_{\boldsymbol{\mu}  \sim Dir(\boldsymbol{\mu |\alpha} )}}[\log p(y|\boldsymbol{\mu})]}\\
				{ - {\lambda _t}{D_{HD}}[Dir(\boldsymbol{\mu |\widetilde \alpha} ||Dir(\boldsymbol{\mu} |[1, \cdots ,1])]}
		\end{array}} \right).
	\end{array}
\end{equation}

The second term of loss function:
\begin{equation}
	\label{eqn_32}
	\begin{array}{l}
		{L^{pseudo}}\left( {\left\{ {x_n^m} \right\}_{m = 1}^M,{y_n}} \right)\\
		\begin{array}{*{20}{c}}
			{}
		\end{array} = \left( {\begin{array}{*{20}{c}}
				{{\mathbb{E}_{{\boldsymbol{\mu ^p}} \sim Dir({\boldsymbol{\mu} ^p}\mid {\boldsymbol{\alpha} ^p})}}[\log p(y|{\boldsymbol{\mu} ^p})]}\\
				{ - {\lambda _t}{D_{HD}}[Dir({\boldsymbol{\mu} ^p}\mid {{\boldsymbol{\widetilde \alpha }^p}}||Dir({\boldsymbol{\mu} ^p}\mid [1, \cdots ,1])]}
		\end{array}} \right).
	\end{array}
\end{equation}

The third term of loss function:
\begin{equation}
	\label{eqn_33}
	\begin{array}{l}
		{L^m}\left( {{x^m},y} \right)\\
		\begin{array}{*{20}{c}}
			{}&{}
		\end{array} = \left( {\begin{array}{*{20}{l}}
				{{\mathbb{E}_{{q_\theta }({{\boldsymbol{\mu }}^m}\mid {x^m})}}[\log p(y|{{\boldsymbol{\mu }}^m})]}\\
				{ - {\lambda _t}HD[D({{\boldsymbol{\mu }}^m}\mid {{\boldsymbol{\alpha }}^m})||D({{\boldsymbol{\mu }}^m}\mid [1, \cdots ,1])]}
		\end{array}} \right).
	\end{array}
\end{equation}

The primary component in the objective function corresponds to the variational objective function for $M$ integrated modalities. Essentially, this variational objective function involves integrating the traditional cross-entropy loss over a simplex defined by the Dirichlet function. The secondary component serves as a prior constraint to ensure the creation of a more plausible Dirichlet distribution. In essence, the primary variational objective function assesses the model's performance by comparing its predictions to the true labels while imposing constraints on the generation of a more sensible Dirichlet distribution.

The second component within the objective function represents the variational objective function for $M$ integrated pseudo-modalities. The third component within the objective function is focused on deriving the Dirichlet distribution for each individual modality. For a specific modality denoted as “$m$”, its loss function can be formulated as previously described. And the overview of the uncertainty estimation via Hölder divergence for multi-view representation learning is shown in Fig. \ref{fig_1}. And the algorithm is shown in \ref{alg:spl}.

\section{Experiments}

In this section, we conduct experiments across diverse scenarios to comprehensively evaluate our algorithm. Specifically, we apply our algorithm to a variety of multi-view classification tasks, including RGB-D scene recognition, using four real-world multi-view datasets.

\subsection{Datasets}
\paragraph{Classification Datasets} To evaluate the performance of our model on multi-view classification tasks, we utilize the following datasets: 1. \textbf{SUNRGBD \cite{pp46}}: The SUN RGB-D dataset includes 4,845 training samples, 3,000 testing samples, and 24,869 samples used for combined training and testing across 19 scene categories. 2. \textbf{NYUDV2 \cite{pp45}}: NYUD2 is an RGB-D dataset with 1,449 image pairs, reorganized into 10 classes, with 795 samples for training and 654 for testing. 3. \textbf{ADE20K \cite{pp47}}: ADE20K is a semantic segmentation dataset with over 20,000 images across 150+ categories, reorganized into 10 groups, with 795 samples for training and 654 for testing. 4. \textbf{ScanNet \cite{pp48}}: ScanNet consists of 1,513 indoor scenes with point cloud data, covering 21 object categories, with 1,201 scenes used for training and 312 for testing. 5. \textbf{OrganAMNIST\cite{zz69,zz70}}: OrganAMNIST is a medical imaging dataset containing four classes, including 97,477 training images, 10,832 validation images, and 1,000 test images, each labeled accordingly.

\paragraph{Clustering Datasets} In addition to classification tasks, our model's performance in clustering tasks is evaluated using three multi-view datasets: 1. \textbf{MSRC-V1 \cite{pp69}}: This image dataset contains eight object classes, each with 30 images. Following \cite{pp69}, we select seven classes: trees, buildings, airplanes, cows, faces, cars, and bicycles. 2. \textbf{Caltech101-7 \cite{pp70}}: A subset of Caltech101, this dataset includes images from seven selected classes, as identified in previous work \cite{pp70}. It is commonly used for training and evaluating object recognition algorithms. 3. \textbf{Caltech101-20 \cite{pp70}}: Another subset of Caltech101, this dataset features images from 20 selected classes based on prior research \cite{pp70}, providing a broader range of objects for testing and refining recognition models.

\renewcommand\arraystretch{1.0}
\begin{table}[t]
	\setlength{\belowdisplayskip}{0pt}
	\setlength{\abovedisplayskip}{0pt}
	\setlength{\abovecaptionskip}{0pt}
	\centering
	%\begin{center}
	\scriptsize
	\caption{Quantitative Evaluation of Intra-Class Experimental Results (Accuracy) on NYUD Depth V2, ADE20K, ScanNet, and SUN RGB-D Datasets. This table compares the performance of ETMC and our proposed method, illustrating their accuracy across various datasets.}%\vspace{-0.25cm}
	\setlength{\tabcolsep}{4pt}%26
	\begin{tabular}{p{2.0cm}p{1.5cm}p{1.1cm}p{1.1cm}p{1.1cm}}  %\toprule[1px]
		\toprule [1.0pt]
		% after \\: \hline or \cline{col1-col2} \cline{col3-col4} ...
		{Models }&	Datasets &RGB (\%) &Depth (\%) &Fusion (\%)\\
		\midrule[0.5pt]
		{ETMC \cite{pp6}}&NYUD2&64.91&65.51&72.43\\
		&ADE20K&85.54&85.60&89.78\\
		&ScanNet&90.71&75.89&91.05\\
		&SUN RGB-D&56.64&52.48&60.80\\
            &OrganAMNIST&92.98&93.45&98.46\\
		\midrule[0.5pt]
  	{Ours}&NYUD2&67.92&65.51&73.60\\
		&ADE20K&86.57&86.89&90.87\\
		&ScanNet&92.31&78.08&92.47\\
		&SUN RGB-D&55.76&54.88&62.05\\
            &OrganAMNIST&94.32&94.59&98.78\\
		\bottomrule[1.0pt]
	\end{tabular}
	%\end{center}
	\label{tab01}%\vspace{-0.45cm}
\end{table}

\renewcommand\arraystretch{1.0}
\begin{table}%[!h]
	\setlength{\belowdisplayskip}{0pt}
	\setlength{\abovedisplayskip}{0pt}
	\setlength{\abovecaptionskip}{0pt}
	\centering
	%\begin{center}
	\scriptsize
	\caption{Anti-Noise Experiments (Accuracy) on NYUD Depth V2 Dataset for Classification Tasks. This table presents the accuracy results of anti-noise experiments conducted on the NYUD Depth V2 dataset, demonstrating the model’s performance under various noise conditions for classification tasks.}%\vspace{-0.25cm}
	\setlength{\tabcolsep}{6pt}%26
	\begin{tabular}{p{1.4cm}p{2.0cm}p{1.0cm}p{1.1cm}p{1.1cm}}  %\toprule[1px]
		\toprule [1.0pt]
		% after \\: \hline or \cline{col1-col2} \cline{col3-col4} ...
		{Datasets}&	Noisy &RGB (\%)&Depth (\%)&Fusion (\%)\\
		\midrule[0.5pt]
		{NYUD2}&$\mu=0,\sigma=0.01$&\textbf{63.54}&\textbf{49.24}&\textbf{64.98}\\
		&$\mu=0,\sigma=0.02$&61.14&31.93&60.99\\
        &$\mu=0,\sigma=0.05$&59.94&10.24&42.62\\
		\midrule[0.5pt]
		{SUN RGB-D}&$\mu=0,\sigma=0.01$&\textbf{50.14}&\textbf{30.55}&\textbf{47.41}\\
		&$\mu=0,\sigma=0.02$&45.11&27.38&44.54\\
        &$\mu=0,\sigma=0.05$&41.39&24.07&40.12\\
		\bottomrule[1.0pt]
	\end{tabular}
	%\end{center}
	\label{tab04}%\vspace{-0.45cm}
\end{table}
\subsection{Evaluation Metrics, Purpose of the Experiment} 
 To effectively evaluate the performance of our classification method, accuracy, recall, precision, and $F_1$-score are chosen as our evaluation metrics. These metrics provide comprehensive insights into the model’s performance and are definedas follows: \textit{Accuracy:} Accuracy measures the overall correctness of the classification and is computed as the ratio of the sum of true positives (TP) and true negatives (TN) to the total number of samples: Accuracy $=\frac{T P+T N}{T P+T N+F P+F N}$.
    \textit{Recall:} Recall, also known as sensitivity or true positive rate, measures the ability of the classifier to correctly identify positive instances. It is calculated as the ratio of TP to the sum of TP and false negatives (FN): Recall $=$ $\frac{T P}{T P+F N}$.
   \textit{Precision:} Precision quantifies the accuracy of positive predictions made by the classifier. It is calculated as the ratio of TP to the sum of TP and false positives (FP): Precision $=\frac{T P}{T P+F P}$.
    \textit{ $F_1$-score:} The $F_1$-score is the harmonic mean of precision and recall, providing a balanced measure that combines both metrics. It is calculated as: $F_1$-score $=$ $2 \times \frac{\text { Precision } \times \text { Recall }}{\text { Precision }+ \text { Recall }}$.

The clustering accuracy (CA) \cite{zz60} is defined as: $\mathrm{CA}=\frac{\sum_{i=1}^n\delta(q_i,\mathrm{map}(p_i))}{n},$ where \(\delta(a,b)=1\) if \(a=b\),and \(\delta(a,b)=0\) otherwise. And \(\mathrm{map}(\cdot)\) is the best permutation mapping that matches the predicted clustering labels to the ground truths.

Considering practical applications, the objectives of this experiment are threefold: (1). Assess the recognition capability of the exploring uncertainty estimation via Hölder divergence for multi-view representation learning (HDMVL) algorithm in more intricate and expansive scenarios, comparing the outcomes with previous experiments conducted on smaller datasets. (2). Examine the potential of Hölder divergence to improve the classification performance of the HDMVL algorithm. Additionally, explore whether fine-tuning Hölder divergence parameters can enhance the model's performance across diverse datasets. (3). Investigate the impact of uncertainty analysis on refining the classification performance of the model in multi-class classification and clustering tasks that encompass multi-view data.

\subsection{Data Preprocessing} Merge and preprocess the samples from the mentioned datasets. In multi-view datasets, images at specific angles typically comprise both color RGB images and depth images. Prior to training, it is necessary to concatenate the image data at specific angles to streamline the classification process.

\renewcommand\arraystretch{1.0}
\begin{table}[t]
	\setlength{\belowdisplayskip}{0pt}
	\setlength{\abovedisplayskip}{0pt}
	\setlength{\abovecaptionskip}{0pt}
	\centering
	%\begin{center}
	\scriptsize
	\caption{Backbone Evaluation of Intra-Class Accuracy on NYUD Depth V2, ADE20K, ScanNet, and SUN RGB-D Datasets. This table highlights the performance of backbone models in distinguishing classes within each dataset, providing insights into model accuracy and robustness.}%\vspace{-0.25cm}
	\setlength{\tabcolsep}{1.5pt}%26
	\begin{tabular}{p{3.5cm}p{1.5cm}p{1.0cm}p{1.1cm}p{1.1cm}}  %\toprule[1px]
		\toprule [1.0pt]
		% after \\: \hline or \cline{col1-col2} \cline{col3-col4} ...
		{Backbone }&	Datasets &RGB (\%) &Depth (\%) &Fusion (\%)\\
		\midrule[0.5pt]
		{ResNet-18 \cite{pp49}}&NYUD2&67.92&65.51&73.60\\
		&ADE20K&86.57&\textbf{86.89}&90.87\\
		&ScanNet&\textbf{92.31}&78.08&\textbf{92.47}\\
		&SUN RGB-D&55.76&54.88&62.10\\
		\midrule[0.5pt]
  	{Mamba-B/32} \cite{pp71} &NYUD2&64.31&64.91&72.59\\
		&ADE20K&85.66&\textbf{84.72}&88.93\\
		&ScanNet&\textbf{91.86}&79.43&\textbf{92.26}\\
		&SUN RGB-D&52.33&54.18&62.31\\
		\midrule[0.5pt]
		{Vit-B/32} \cite{pp72}&NYUD2&72.44&50.15&74.10\\
		&ADE20K&89.64&\textbf{76.55}&91.68\\
		&ScanNet&\textbf{93.76}&70.34&\textbf{94.03}\\
		&SUN RGB-D&60.21&56.59&63.26\\
		\bottomrule[1.0pt]
	\end{tabular}
	%\end{center}
	\label{tab03}%\vspace{-0.45cm}
\end{table}

\subsection{Model Architecture} During the study, we use three different network architectures. The ResNet-18 \cite{pp49} pretrained on the ImageNet \cite{pp50} served as our foundational framework. ResNet-18 is a deep residual neural network comprising 18 layers. The second is the Mamba model \cite{pp71} that performs well in long sequence modeling tasks. Mamba alleviates the modeling constraints of convolutional neural networks through global field of perception and dynamic weighting, and provides advanced modeling capabilities similar to transformers. The last is vision transformer (ViT) \cite{pp72}, which applies a direct transformer to sequences of image patches. Training is performed on a computer equipped with an Intel(R) Core(TM) i9-11900KF CPU @ 3.50GHz, 64.00 GB RAM, and a 4090Ti GPU. The input image size is standardized to $256 \times 256$ and further randomly cropped to $224 \times 224$. The Adam \cite{pp51} optimizer is used to training the neural networks with weight and learning rate decay. In the case of HDMVL, the pseudo-view is generated by directly connecting the output of the three backbone networks, where we fix the Hölder exponent at 1.7. All experiments are implemented using PyTorch \cite{pp52}.

\begin{table*}[t]%[htbp]
    \centering
    \small
    \caption{Quantitative Evaluation of Inter-Class Experimental Results(\%) on the NYUD Depth V2 and SUN RGB-D Datasets Compared with State-of-the-Art Methods. This table provides a detailed comparison of inter-class Accuracy (Acc), Precision (Pre), Recall (Rec), and F1-Score for various methods on the NYUD Depth V2 and SUN RGB-D datasets, highlighting the performance of our approach relative to state-of-the-art techniques.}
    \resizebox{\textwidth}{!}{ % 动态调整表格宽度
        \begin{tabular}{l|l|ccc|ccc|ccc|ccc|ccc|ccc}
            \toprule[1pt]
            \multirow{2}{*}{Dataset} & \multirow{2}{*}{Metric}
            & \multicolumn{3}{c|}{Ours}
            & \multicolumn{3}{c|}{TrecgNet \cite{pp53}} 
            & \multicolumn{3}{c|}{G-L-SOOR \cite{pp54}} 
            & \multicolumn{3}{c|}{CBCL-RGBD \cite{pp55}} 
            & \multicolumn{3}{c|}{CMPT \cite{pp56}}
            & \multicolumn{3}{c}{CNN-randRNN \cite{pp57}}  \\ % 去掉最后一个模型的分割线
            \cmidrule(lr){3-5} \cmidrule(lr){6-8} \cmidrule(lr){9-11}
            \cmidrule(lr){12-14} \cmidrule(lr){15-17} \cmidrule(lr){18-20}
            & & RGB & Depth & Fusion & RGB & Depth & Fusion & RGB & Depth & Fusion 
            & RGB & Depth & Fusion & RGB & Depth & Fusion & RGB & Depth & Fusion \\
            \midrule[0.5pt]
            NYUD Depth V2 & Acc       & 63.60 & \textbf{65.50} & \textbf{73.60}& 64.80 & 57.70 & 69.20 & 64.20 & 62.30 & 67.40 & 66.40 & 49.50 & 70.90 & 66.10 & 64.10 & 71.80 & \textbf{69.10} & 48.30 & 69.10  \\
                          & Precision & 64.71 & 66.19 & 73.63 & -- & -- & -- & -- & -- & -- & -- & -- & -- & -- & -- & -- & -- & -- & -- \\
                          & Rec       & 65.81 & 64.46 & 73.49 & -- & -- & -- & -- & -- & -- & -- & -- & -- & -- & -- & -- & -- & -- & -- \\
                          & F1 Score  & 64.55 & 64.58 & 72.69 & -- & -- & -- & -- & -- & -- & -- & -- & -- & -- & -- & -- & -- & -- & -- \\
            \midrule[0.5pt]
            SUN RGB-D     & Acc       & 55.80 & \textbf{54.90} & \textbf{62.10}& 50.60 & 47.90 & 56.70 & 50.50 & 44.10 & 55.50 & 48.80 & 37.30 & 59.50 & 54.20 & 49.30 & 59.80 & \textbf{58.50} & 50.10 & 60.70  \\
                          & Precision & 55.80 & 51.37 & 59.94 & -- & -- & -- & -- & -- & -- & -- & -- & -- & -- & -- & -- & -- & -- & -- \\
                          & Rec       & 56.06 & 51.34 & 60.35 & -- & -- & -- & -- & -- & -- & -- & -- & -- & -- & -- & -- & -- & -- & -- \\
                          & F1 Score  & 54.15 & 50.65 & 58.45 & -- & -- & -- & -- & -- & -- & -- & -- & -- & -- & -- & -- & -- & -- & -- \\
            \bottomrule[1.0pt]
        \end{tabular}
    }
    \label{tab02}
\end{table*}

\paragraph{Comparison Experiments for Classification} TrecgNet \cite{pp53}: Enhances scene recognition models by leveraging both RGB and depth modalities for improved robustness and performance. G-L-SOOR \cite{pp54}: Focuses on RGB-D scene recognition, emphasizing spatial object-to-object relations in image representations to enhance model effectiveness. CBCL-RGBD \cite{pp55}: Introduces a centroid-based concept learning approach for RGB-D indoor scene classification. CMPT \cite{pp56}: Proposes a Cross-Modal Pyramid Translation method for RGB-D scene recognition, aiming to enhance cross-modal feature learning. CNN-randRNN \cite{pp57}: Integrates Convolutional Neural Networks (CNNs) and Random Recurrent Neural Networks (RNNs) for multi-level analysis in RGB-D object and scene recognition. ETMC \cite{pp6}: Introduces the ETMC algorithm, incorporating dynamic evidential fusion and a pseudo-view concept, aiming to enhance multi-view classification and improve reliability by evaluating uncertainty based on subjective logic theory and the Dempster-Shafer evidence theory.

\paragraph{Comparison Experiments for Clustering} We conduct performance comparisons on multi-view clustering using several popular state-of-the-art methods, including SWMC \cite{ppp55}, MLAN \cite{ppp56}, MSC-IAS \cite{pppp57}, MCGC \cite{ppp58}, BMVC \cite{ppp59}, and DSRL \cite{ppp60}.

\subsection{Experimental Analysis} For multi-view classification, accuracy (ACC) stands out as a pivotal metric. Our objective in multi-view classification is to accurately classify scenes within the dataset using the network for subsequent analysis.

\paragraph{Intra-Class Experimental Results} Intra-class experiments entail testing the ETMC model and the HDMVL model with various hyper-parameters on real-world scene datasets. In this series of experiments, four distinct datasets are employed to evaluate classification performance in intricate scenarios. The experimental results are presented in Table \ref{tab03}.

\renewcommand\arraystretch{1.0}
\begin{table}%[!h]
	\setlength{\belowdisplayskip}{0pt}
	\setlength{\abovedisplayskip}{0pt}
	\setlength{\abovecaptionskip}{0pt}
	\centering
	%\begin{center}
	\scriptsize
	\caption{Clustering Performance of Various Multi-View Clustering Methods Across Three Datasets. This table summarizes the performance of different multi-view clustering methods, evaluated on three distinct datasets, highlighting their comparative effectiveness in terms of clustering accuracy.}%\vspace{-0.25cm}
	\setlength{\tabcolsep}{2pt}%26
	\begin{tabular}{p{1.5cm}p{3.3cm}p{1.0cm}p{1.1cm}p{1.1cm}}  %\toprule[1px]
		\toprule [1.0pt]
		% after \\: \hline or \cline{col1-col2} \cline{col3-col4} ...
		{Datasets }&	Methods &RGB (\%) &Depth (\%) &Fusion (\%)\\
		\midrule[0.5pt]
		{Caltech101-7}&MLAN \cite{ppp56}&-&-&78.00\\
            &SwMC \cite{ppp55}&-&-&66.50\\
            &MCGC \cite{ppp58}&-&-&64.30\\
		&BMVC \cite{ppp59} &-&-&57.90\\
        &MSC-IAS \cite{ppp61} &-&-&71.30\\
        &DSRL \cite{ppp60} &-&-&83.80\\
		&Ours&90.25&89.82&\textbf{96.91}\\
		\midrule[0.5pt]
		{Caltech101-20}&MLAN \cite{ppp56} &-&-&52.60\\
            &SWMC \cite{ppp55} &-&-&54.10\\
            &MCGC \cite{ppp58} &-&-&54.60\\
		&BMVC \cite{ppp59} &-&-&47.40\\
        &MSC-IAS \cite{ppp61} &-&-&41.90\\
        &DSRL \cite{ppp60} &-&-&72.90\\
		&Ours&68.97&70.36&\textbf{92.59}\\
		\midrule[0.5pt]
		{MSRC-v1}&MLAN \cite{ppp56} &-&-&68.10\\
            &SwMC \cite{ppp55} &-&-&78.60\\
            &MCGC \cite{ppp58}&-&-&75.20\\
		&BMVC \cite{ppp59} &-&-&63.80\\
        &MSC-IAS \cite{ppp61} &-&-&75.20\\
        &DSRL \cite{ppp60} &-&-&83.40\\
		&Ours &98.92&98.51&\textbf{100.00}\\
		\bottomrule[1.0pt]
	\end{tabular}
	%\end{center}
	\label{tab05}%\vspace{-0.45cm}
\end{table}

\renewcommand\arraystretch{1.0}
\begin{table}[h]
	\setlength{\belowdisplayskip}{0pt}
	\setlength{\abovedisplayskip}{0pt}
	\setlength{\abovecaptionskip}{0pt}
	\centering
	%\begin{center}
	\scriptsize
	\caption{Hyperparameter Experiments on NYUD Depth V2, ADE20K, ScanNet, and SUN RGB-D Datasets. This table presents the results of hyperparameter experiments, exploring the performance of our approach across four benchmark datasets: NYUD Depth V2, ADE20K, ScanNet, and SUN RGB-D.}%\vspace{-0.25cm}
	\setlength{\tabcolsep}{6pt}%26
	\begin{tabular}{p{1.5cm}p{1.7cm}p{1.0cm}p{1.1cm}p{1.1cm}}  %\toprule[1px]
		\toprule [1.0pt]
		% after \\: \hline or \cline{col1-col2} \cline{col3-col4} ...
		{Datasets}&	Models &RGB (\%) &Depth (\%) &Fusion (\%)\\
		\midrule[0.5pt]
		{NYUD2 }&Ours ($\alpha=1.2$)&66.11&\textbf{65.66}&72.29\\
		&Ours ($\alpha=1.3$)&66.27&65.06&72.44\\
		&Ours ($\alpha=1.7$)&\textbf{67.92}&65.51&\textbf{73.60}\\
		&Ours ($\alpha=1.9$)&65.06&63.40&72.44\\
		&Ours ($\alpha=2.0$)&65.51&65.36&72.59\\
		\midrule[0.5pt]
		{ADE20K}&Ours ($\alpha=1.1$)&86.05&86.95&90.62\\
		&Ours ($\alpha=1.5$)&86.31&86.89&90.55\\
		&Ours ($\alpha=1.7$)&86.57&86.31&\textbf{90.87}\\
		&Ours ($\alpha=1.8$)&\textbf{86.76}&\textbf{86.89}&90.55\\
		&Ours ($\alpha=2.0$)&86.50&86.25&90.62\\
		\midrule[0.5pt]
		{ScanNet} &Ours ($\alpha=1.2$)&92.34&78.13&92.21\\
		&Ours ($\alpha=1.3$)&92.47&78.63&92.17\\
		&Ours ($\alpha=1.5$)&92.03&78.28&92.21\\
		&Ours ($\alpha=1.7$)&92.17&78.00&92.24\\
		&Ours ($\alpha=1.8$)&\textbf{92.31}&\textbf{78.08}&\textbf{92.47}\\	
		\midrule[0.5pt]
		{SUN RGB-D}&Ours ($\alpha=1.2$)&56.30&53.44&61.42\\
		&Ours ($\alpha=1.5$)&\textbf{56.98}&53.66&61.58\\
		&Ours ($\alpha=1.6$)&56.47&54.71&61.36\\
		&Ours ($\alpha=1.7$)&55.76&\textbf{54.88}&\textbf{62.10}\\
		&Ours ($\alpha=1.8$)&56.58&53.42&61.77\\
		\bottomrule[1.0pt]
	\end{tabular}
	%\end{center}
	\label{tab07}%\vspace{-0.45cm}
\end{table}

In the intra-class experiments, we assess the HDMVL model's performance across five multi-class datasets, comparing it with the HDMVL model. During testing, we evaluate the classification accuracy of individual modalities separately as well as in their fused form. The experimental results are summarized in Table \ref{tab01}.

For the two 10-class datasets, NYUD Depth V2 and ADE20K, the HDMVL model demonstrate superior performance, achieving fusion modality accuracies of 73.64\% and 90.87\%, respectively—an improvement of 1.21\% and 1.09\% over the ETMC model. Notably, accuracy for individual modalities also increased after incorporating Hölder divergence, particularly in the color RGB modality of the NYUD Depth V2 dataset, where recognition accuracy improved by 3.01\%. This improvement is even more pronounced in the 16-class ScanNet and 19-class SUN RGB-D datasets. The fusion modality accuracy on the SUN RGB-D reached 62.10\%, surpassing the ETMC model by 1.25\%. In terms of performance on the OrganAMNIST medical imaging dataset, the HDMVL model consistently outperforms the ETMC model in both single-modality and fusion-modality scenarios. The likely reason for this improvement is that Hölder divergence, when apply to multi-class data, can more accurately identify the data features of each category.

These results suggest that HDMVL maintains high accuracy in more complex scenarios with a greater number of classes, achieving improved classification performance through enhancements to the objective function based on the Hölder index.

\paragraph{Inter-Class Experimental Results} Inter-class experiments entail a comparison between the HDMVL and pre-existing algorithms that have undergone experimentation on the datasets employ in this study. Subsequent to analyzing the experimental results, NYUD Depth V2 and SUN RGB-D, the two datasets with the most extensive experimentation, are chosen for further scrutiny.

In this study, we conduct a comprehensive comparison of our proposed HDMVL with the current state-of-the-art methods using the NYUD Depth V2 and SUN RGB-D datasets. The results clearly demonstrate that our model outperforms these methods on both datasets. Notably, in the classification of fused modalities, our model adeptly integrates information features from RGB and Depth modalities in a highly rational manner, achieving the highest accuracy among similar models at 73.6\% and 62.1\%, respectively.

The experimental findings underscore the positive impact of uncertainty analysis on enhancing the accuracy of multi-view classification models, particularly in the context of fused modalities. Uncertainty analysis enables the model to discern more accurately which modality's information is reliable and precise in a given scene. Consequently, the model places greater emphasis on the information from this modality during fusion, leading to improved results. Furthermore, the refinement of the objective function based on the Hölder divergence enhances the specificity and granularity of uncertainty analysis results, contributing to a further boost in the model's overall performance. The experimental results are presented in Table \ref{tab02}. %The quantitative assessment visualization is depicted in Fig. \ref{fig_2}, while Fig. \ref{fig_3} illustrates the growth of RGB, depth, and fusion modalities on NYUD Depth V2, ADE20K, ScanNet, and SUN RGB-D datasets.

\begin{figure*}%[h]
	\centering
	\includegraphics[width=\linewidth]{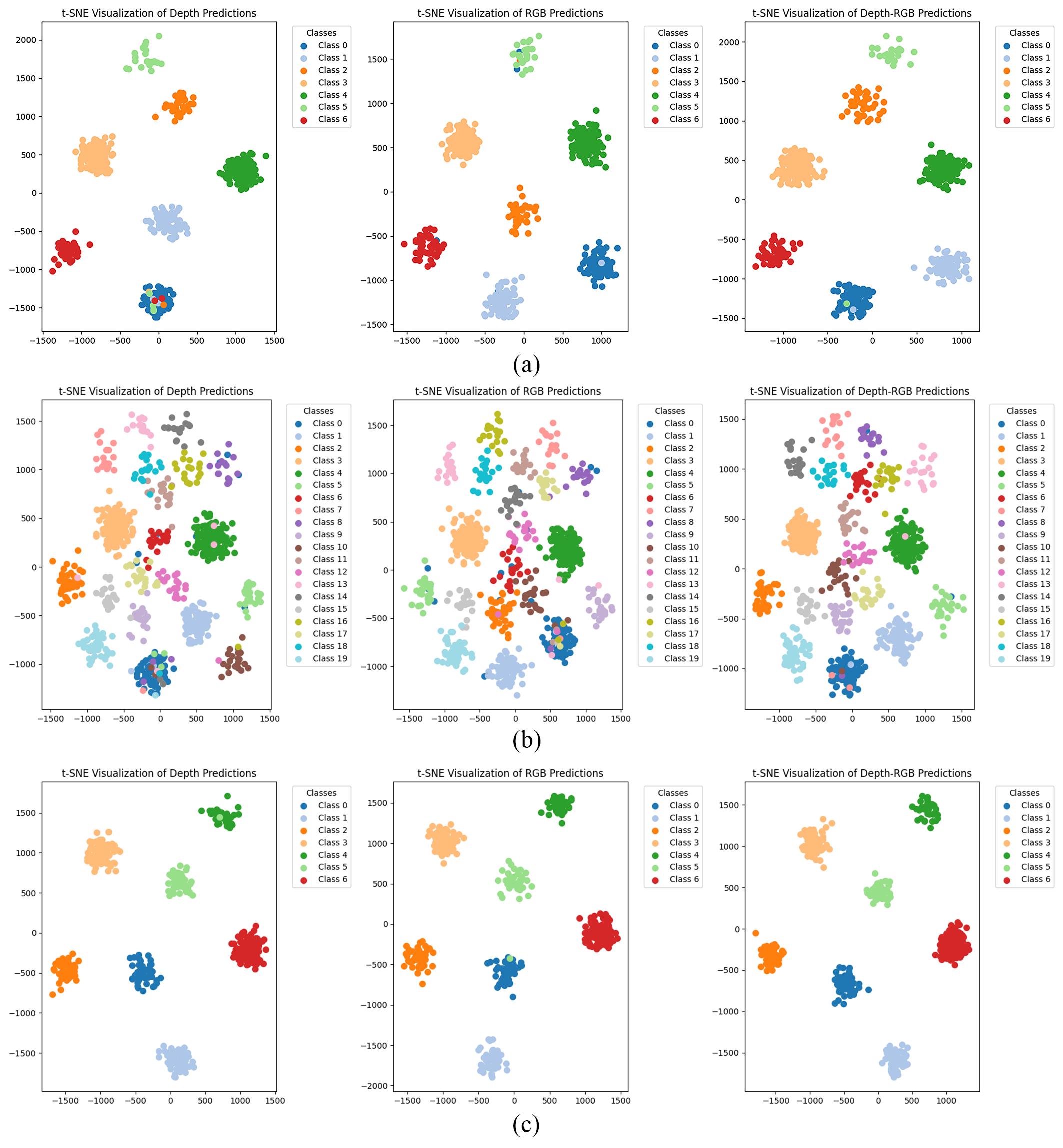}
	\caption{Overview of uncertainty estimation using Hölder divergence for multi-view representation learning. The figure presents t-SNE visualizations of multi-view clustering results across different datasets: (a) Caltech101-7, (b) Caltech101-20, and (c) MSRC-V1. These visualizations demonstrate how our model's uncertainty quantification, based on Hölder divergence, improves clustering performance. Additionally, the figure provides a comparative analysis, highlighting the enhanced separation of clusters and the robustness of our approach across diverse datasets.}
	\label{fig_2}
\end{figure*}
\paragraph{Inter-Class Experimental Results} On the basis of the above experiments, we carry out experiments of different network architectures. The performance of ResNet \cite{pp49}, Mamba \cite{pp71} and VIT \cite{pp72} on four multi-class datasets is tested in Table \ref{tab03}.

We observe that the model maintains strong classification performance after changing the network architecture, particularly when the backbone is replaced with VIT \cite{pp72}, resulting in higher accuracy compared to the other two architectures. This improvement suggests that the global attention mechanism in VIT better captures image features, leading to more reliable classification results. These findings demonstrate that our method is adaptable to different network architectures. Additionally, we validate the model's robustness on noisy datasets. Detailed results are presented in Table \ref{tab04}. Gaussian noise with a mean of 0 and variances of [0.01, 0.02, 0.05] is injected into two life scenario datasets, a and b, respectively. The HDMVL model is then trained with a Hölder index of 1.7.

\paragraph{Clustering Experimental Results} Table \ref{tab05} compares the clustering performance of HDMVL with several state-of-the-art methods on the Caltech101-7, Caltech101-20, and MSRC-v1 datasets. Overview of uncertainty estimation using Hölder divergence for multi-view representation learning is shown in Fig. \ref{fig_2}. t-SNE visualizations of multi-view clustering results on diverse datasets: (a) Caltech101-7, (b) Caltech101-20, and (c) MSRC-V1. These results demonstrate that our model's uncertainty quantification enhances clustering performance and provides a comparative analysis of the outcomes. The results show that most multi-view clustering methods perform worse than ours. Notably, the HDMVL model achieves a higher clustering effect using a single mode compared to other methods using both modes. When utilizing multiple modes, the HDMVL model significantly outperforms the other methods. Although HDMVL is not specifically designed for clustering tasks, it successfully handles these scenarios, demonstrating its robust learning capability even when trained on clustering datasets. 
\renewcommand\arraystretch{1.0}
\begin{table}[t]
	\setlength{\belowdisplayskip}{0pt}
	\setlength{\abovedisplayskip}{0pt}
	\setlength{\abovecaptionskip}{0pt}
	\centering
	%\begin{center}
	\scriptsize
	\caption{Quantitative Evaluation Results of the Ablation Study (Accuracy) on the ADE20K Dataset. This table presents the accuracy results from an ablation study, comparing the performance of KL divergence and Hölder divergence on the ADE20K dataset.}%\vspace{-0.25cm}
    \resizebox{0.5\textwidth}{!}{
	\begin{tabular}{p{0.35cm}p{0.01
    cm}p{0.01cm}p{0.01cm}p{1.0cm}p{1.2cm}p{1.2cm}}  %\toprule[1px]
		\toprule [0.8pt]
		% after \\: \hline or \cline{col1-col2} \cline{col3-col4} ...
		{Hölder}&KL&CS&JS&RGB (\%)&Depth (\%)&Fusion (\%)\\
		\midrule[0.5pt]
		{$\bullet$}&$\circ$&$\circ$&$\circ$&86.57&86.31&\textbf{90.87}\\
		$\circ$&$\bullet$&$\circ$&$\circ$&85.54&85.60&89.78\\       $\circ$&$\circ$&$\bullet$&$\circ$&72.44&\textbf{89.34}&90.21\\  
$\circ$&$\circ$&$\circ$&$\bullet$&\textbf{87.63}&85.93&89.02\\  
		\bottomrule[1.0pt]
	\end{tabular}
    }
	%\end{center}
	\label{tab06}%\vspace{-0.45cm}
\end{table}
\subsection{Ablation Study}

To further clarify, the ADE20K dataset is selected as the experimental basis for training the classification model to evaluate the impact of Hölder divergence on both individual modality recognition and fused modality recognition. In addition, we test the performance of Kullback-Leibler divergence (KL) \cite{pp4}, Cauchy-Schwarz divergence (CS) \cite{pp3}, and Jensen–Shannon divergence (JS) \cite{zz71} on the ADE20K dataset. The results, as shown in Table \ref{tab06}, demonstrate a significant improvement in accuracy after incorporating Hölder divergence into the model. This enhancement is particularly pronounced in individual modality recognition, where the model's ability to accurately classify distinct modalities saw a notable increase. Additionally, in fused modality recognition, where information from multiple modalities is integrated, the model achieves higher accuracy compared to its original version. To assess the effect of the Hölder exponent on model performance, we conduct tests on several different datasets, as presented in Table \ref{tab07}. The results indicate that the highest accuracy in the fusion mode of the classification model occurs when the Hölder exponent is 1.7. Deviating from this value, either lower or higher, leads to a decline in fusion mode accuracy. To accurately measure the upper limit of the HDMVL model, we use two intelligent hyperparameter selection methods in addition to manual tuning based on experience: (1) treating the hyperparameters as learnable parameters and (2) conducting automated hyperparameter searches. However, the experimental results from both methods deviate significantly from expectations. As a result, we primarily rely on manual fine-tuning based on experience to optimize the model parameters.

These findings underscore the positive impact of Hölder divergence on the model's classification capabilities, both for individual modalities and in scenarios involving the fusion of diverse modalities. The implications extend beyond the ADE20K dataset, suggesting potential improvements in classification performance and generalization across various multi-class datasets, particularly in situations with limited sample sizes.

It is evident that most multi-view clustering methods perform worse than ours. Additionally, the HDMVL model achieves a higher clustering effect using a single mode than other methods using both modes. The HDMVL model significantly outperforms other methods when using multiple modes for clustering. Although HDMVL's uncertainty estimation method is not specifically designed for clustering tasks, it effectively handles these scenarios, indicating that HDMVL possesses robust learning capabilities even when trained with clustering datasets.

\section{Conclusion}
This study presents an uncertainty-aware variational Dirichlet learning approach to tackle challenges in multi-view representation learning. By incorporating subjective logic, the DS-combination rule, and Hölder divergence between Dirichlet distributions, the methodology significantly enhances recognition performances across a wide range of multi-modal benchmarks. Extensive experimental results confirm the approach's theoretical soundness and practical robustness, demonstrating improved performance in complex datasets and the effectiveness of Hölder divergence in uncertainty measurement.

\appendices
\section{The Rationale for Employing Hölder Divergence}
\label{appendices}
HD can be analytically computed for exponential family distributions. Fortunately, based on the analysis above, the Dirichlet distribution also falls under the category of exponential family distributions, ensuring practical training and exhibiting favorable properties. In following section, we provide the analytical expression of HD for two Dirichlet distributions.

HD is introduced for closed-form optimization, offering a distinct advantage over KLD, which lacks closed-form solutions for several distributions. It provides closed-form expressions of HPD for conic and affine exponential families as follows:
\newtheorem{lemma}{\bf{Lemma}}
\begin{lemma} 
	\label{lemma_1}
	(\textbf{HPD and PHD for Conic or Affine Exponential Family}) \cite{pp3}. For distributions $p(x;\theta_p)$ and $p(x;\theta_q)$ that are part of the same exponential family with conic or affine natural parameter space, both the HPD and PHD can be expressed in closed-form:
	\begin{equation}
		\label{equ_13}
          D_\alpha ^H(p:q) = \frac{1}{\alpha }F(\alpha {\theta _p}) + \frac{1}{\beta }F(\beta {\theta _q}) - F({\theta _p} + {\theta _q}),
	\end{equation}
	where the log-normalizer, denoted as $F(\theta)$, is a strictly convex function also referred to as the cumulant generating function.
\end{lemma}

\newtheorem{theorem}{\bf{Theorem}}
\begin{theorem}
	\label{theorem_3}
	For Dirichlet distributions $p(x;\theta_\theta)$ and $p(x;\theta_\mu)$ that are part of the same exponential family with conic or affine natural parameter space, the Hölder pseudo-divergence is as follows:
	\begin{equation}
		\label{eqn_14}
		\begin{aligned}
			D_\alpha^{\mathrm{H}}(p: q) & =\frac{1}{\alpha} F\left(\alpha \theta\right)+\frac{1}{\beta} F\left(\beta \mu\right)-F\left(\theta+\mu\right),
		\end{aligned}
	\end{equation}  
	where $\bar{\alpha}=\frac{\alpha}{\alpha-1}$, and $F(\theta ) = \sum\limits_k {\log } \Gamma \left( {{\theta _k} + 1} \right) - \log \Gamma \left( {\sum\limits_k {({\theta _k} + 1)} } \right)$.
\end{theorem}
\begin{proof}
	Hölder pseudo-divergences, using Lemma. \ref{lemma_1}, for the first term, we can derive the following inferences:
	\begin{equation}
		\label{eqn_16}
		\begin{array}{l}
			\frac{1}{\alpha }F(\alpha \theta ) = \frac{1}{\alpha }\left[ {\sum\limits_k {\log } \Gamma \left( {\alpha {\theta _k} + 1} \right) - \log \Gamma \left( {\sum\limits_k {\left( {\alpha {\theta _k} + 1} \right)} } \right)} \right]\\
			\begin{array}{*{20}{c}}
				{}&{}&{}
			\end{array} = \frac{1}{\alpha }\left[ {\begin{array}{*{20}{c}}
					{k\log \alpha  + \sum\limits_k {\log } {\theta _k} + \sum\limits_k {\log } \Gamma \left( {\alpha {\theta _k}} \right)}\\
					{ - \log \Gamma \left( {\sum\limits_k \alpha  {\theta _k}} \right)\begin{array}{*{20}{c}}
							{}&{}&{}&{}&{}
					\end{array}}\\
					{ - \sum\limits_k {\log } \left( {\sum\limits_k \alpha  {\theta _k} + k - 1} \right)\begin{array}{*{20}{c}}
							{}&{}&{}
					\end{array}}
			\end{array}} \right],
		\end{array}
	\end{equation}
	\begin{equation}
		\label{eqn_17}
		\begin{array}{l}
			\sum\limits_k {\log } \Gamma \left( {\alpha {\theta _k} + 1} \right) = \sum\limits_k {\log } \left[ {\alpha {\theta _k}\Gamma \left( {\alpha {\theta _k}} \right)} \right]\\
			\begin{array}{*{20}{c}}
				{}&{}
			\end{array} = \sum\limits_k {\left[ {\log \alpha {\theta _k} + \log \left[ {\left( {\alpha {\theta _k}} \right)} \right]} \right.} \\
			\begin{array}{*{20}{c}}
				{}&{}
			\end{array} = \sum\limits_k {\left[ {\log \alpha  + \log {\theta _k} + \log \left[ {\left( {\alpha {\theta _k}} \right)} \right]} \right.} \\
			\begin{array}{*{20}{c}}
				{}&{}
			\end{array} = k\log \alpha  + \sum\limits_k {\log } {\theta _k} + \sum\limits_k {\log } \Gamma \left( {\alpha {\theta _k}} \right),
		\end{array}
	\end{equation}
	\begin{equation}
		\label{eqn_18}
		\begin{array}{l}
			\log \Gamma \left( {\sum\limits_k {\left( {\alpha {\theta _k} + 1} \right)} } \right) = \log \Gamma \left( {\sum\limits_k \alpha  {\theta _k} + k} \right)\\
			\begin{array}{*{20}{c}}
				{}&{}
			\end{array} = \log \left[ {\begin{array}{*{20}{c}}
					{\Gamma \left( {\sum\limits_k \alpha  {\theta _k}} \right)\left( {\sum\limits_k \alpha  {\theta _k}} \right)\left( {\sum\limits_k \alpha  {\theta _k} + 1} \right)}\\
					{ \ldots \left( {\sum\limits_k \alpha  {\theta _k} + k - 1} \right)\begin{array}{*{20}{c}}
							{}&{}&{}&{}&{}
					\end{array}}
			\end{array}} \right]\\
			\begin{array}{*{20}{c}}
				{}&{}
			\end{array} = \log \Gamma \left( {\sum\limits_k \alpha  {\theta _k}} \right) + \sum\limits_k {\log } \left( {\sum\limits_k \alpha  {\theta _k} + k - 1} \right).
		\end{array}
	\end{equation}
	For the second term, we can deduce the following conclusions:
	\begin{equation}
		\label{eqn_19}
		\begin{array}{l}
			\frac{1}{\beta }F(\beta \mu ) = \frac{1}{\beta }\left[ {\sum\limits_k {\log } \Gamma \left( {\beta {\mu _k} + 1} \right) - \log \Gamma \left( {\sum\limits_k {\left( {\beta {\mu _k} + 1} \right)} } \right)} \right]\\
			\begin{array}{*{20}{c}}
				{}&{}&{}&{}
			\end{array} = \frac{1}{\beta }\left[ {\begin{array}{*{20}{c}}
					{k\log \beta  + \sum\limits_k {\log } {\mu _k} + \sum\limits_k {\log } \Gamma \left( {\beta {\mu _k}} \right)}\\
					{ - \log \Gamma \left( {\sum\limits_k \beta  {\mu _k}} \right)\begin{array}{*{20}{c}}
							{}&{}&{}&{}&{}
					\end{array}}\\
					{ - \sum\limits_k {\log } \left( {\sum\limits_k \beta  {\mu _k} + k - 1} \right)\begin{array}{*{20}{c}}
							{}&{}&{}
					\end{array}}
			\end{array}} \right].
		\end{array}
	\end{equation}
	Regarding the third term $F\left( {\theta  + \mu } \right)$, we can draw the following conclusions:
	\begin{equation}
		\label{eqn_20}
		\begin{array}{l}
			\sum\limits_k {\log } \Gamma \left( {{\theta _k} + {\mu _k} + 1} \right) - \log \Gamma \left( {\sum\limits_k {\left( {{\theta _k} + {u_k} + 1} \right)} } \right)\\
			\begin{array}{*{20}{c}}
				{}&{}
			\end{array} = \left[ {\begin{array}{*{20}{c}}
					{\sum\limits_k {\log } \left( {{\theta _k} + {\mu _k}} \right) + \sum\limits_k {\log } \Gamma \left( {{\theta _k} + {\mu _k}} \right)}\\
					{ - \log \Gamma \left( {\sum\limits_k {\left( {{\theta _k} + {\mu _k}} \right)} } \right)\begin{array}{*{20}{c}}
							{}&{}&{}&{}
					\end{array}}\\
					{ - \sum\limits_k {\log } \left( {\sum\limits_k {\left( {{\theta _k} + {u_k}} \right)}  + k - 1} \right)\begin{array}{*{20}{c}}
							{}&{}
					\end{array}}
			\end{array}} \right].
		\end{array}
	\end{equation}
\end{proof}

\begin{theorem}
    For variational inference using Dirichlet Models, the HPD provides a tighter ELBO compared to the KLD $D_{\text{KL}}(p \| q)$.
\end{theorem}
\begin{proof}
    For the KL divergence in Dirichlet models, we have:
    \begin{equation}
        D_{\text{KL}}(q(z|x) \| p(z)) = \int q(z|x) \log \frac{q(z|x)}{p(z)} \, dz.
    \end{equation}

For the HPD in Dirichlet models, we have:
\begin{equation}
        \begin{array}{l}
        D_\alpha ^H(q(z|x)||p(z))\\
        \begin{array}{*{20}{c}}
        {}&{}&{}&{}&{}&{}
        \end{array} = \left( \begin{array}{l}
        \frac{1}{\alpha }F(\alpha {\theta _{q(z|x)}}) + \frac{1}{\beta }F(\beta {\theta _{p(z)}})\\
         - F({\theta _{q(z|x)}} + {\theta _{p(z)}})
        \end{array} \right).
        \end{array}
\end{equation}
The ELBO with the HPD becomes:
\begin{equation}
    \text{ELBO}_{\text{H}} = \mathbb{E}_{q(z|x)} [\log p(x|z)] - D_{\alpha}^{H}(q(z|x) \| p(z)).
\end{equation}
To show that the ELBO with the HPD is tighter than the ELBO with the KLD, we need to show that: $\text{ELBO}_{\text{H}} \ge \text{ELBO}_{\text{KL}}$.

Since the HPD is more flexible and tunable through the parameters $\alpha, \beta$, it can better fit the true posterior distribution and reduce the gap between the variational distribution and the true posterior.
\end{proof}

\begin{theorem}
\label{theorem_4}
Using HPD as a regularization term in variational inference with Dirichlet distributions improves model robustness compared to using KLD.
\end{theorem}

\begin{proof}
   In variational inference, the objective function is typically the ELBO:
   \begin{equation}
       \mathcal{L}_{\text{ELBO}} = \mathbb{E}_{q_\theta(z|x)}[\log p_\theta(x|z)] - D_{\text{KL}}(q_\theta(z|x) \| p(z)).
   \end{equation}
   We replace the KLD with HPD, resulting in a new objective function:
   \begin{equation}
       \mathcal{L}_{\text{HPD}} = \mathbb{E}_{q_\theta(z|\mathcal{D})}[\log p_\theta(x|z)] - D_{\alpha}^{\text{H}}(q_\theta(z|x) \| p(z)),
   \end{equation}
 where \( D_{\alpha}^{\text{H}} \) is the regularization term based on HPD, defined as:
 \begin{equation}
     D_\alpha^{\mathrm{H}}(p: q) = \frac{1}{\alpha} F(\alpha \theta_p) + \frac{1}{\beta} F(\beta \theta_q) - F(\theta_p + \theta_q).
 \end{equation}
   For Dirichlet distributions, assume \( p(x; \theta_p) \) and \( q(x; \theta_q) \) are parameterized distributions with parameter vectors \( \theta_p \) and \( \theta_q \).

   Using the definition of HPD, first compute the log-normalizing function \( F(\theta) \) for each distribution and then substitute it into the formula $\mathcal{L}_{\text{HPD}}$:
   \begin{equation}
        {\mathbb{E}_{{q_\theta }(z|x)}}[\log {p_\theta }(x|z)] - \left( \begin{array}{l}
        \frac{1}{\alpha }F(\alpha {\theta _p}) + \frac{1}{\beta }F(\beta {\theta _q})\\
         - F({\theta _p} + {\theta _q})
        \end{array} \right).
   \end{equation}
   HPD provides greater flexibility under different parameters, capturing subtle differences between distributions. This is particularly important for distributions with multimodal characteristics. By optimizing this new objective function, model robustness is enhanced.
\end{proof}

\begin{theorem}
For Dirichlet distributions with significant differences in parameters, the Hölder divergence \( D_\alpha^H \) better captures the differences in distribution modes compared to the KL divergence \( D_{\text{KL}} \).
\end{theorem}

\begin{proof}
   The mode of a Dirichlet distribution \( p \) is given by:
   \begin{equation}
       \text{mode}(x) = \frac{\alpha_i - 1}{\sum_{j=1}^K (\alpha_j - 1)}.
   \end{equation}
   HPD with \( \alpha \neq 1 \) emphasizes different aspects of the distributions compared to KL divergence, particularly capturing the influence of parameters that lead to different modes.
   \begin{equation}
       D_\alpha^H(p \| q) = \frac{1}{\alpha} F(\alpha \theta) + \frac{1}{\beta} F(\beta \mu) - F(\theta + \mu).
   \end{equation}
   When \( \alpha \neq 1 \), the HPD takes into account the distribution’s modes more effectively compared to KLD, especially when \( \mathbf{\alpha} \) and \( \mathbf{\beta} \) differ significantly.
\end{proof}

%%%%%%%%%%%%%%%%%%%%%%%%%%%

%%%%%%%%%%%%%%%%%%%%%%%%%%%
	
	\ifCLASSOPTIONcaptionsoff
	\newpage
	\fi
	
	\bibliographystyle{IEEEtran}
	\bibliography{IEEEabrv,references.bib}

% Generated by IEEEtran.bst, version: 1.14 (2015/08/26)
\begin{thebibliography}{10}
\providecommand{\url}[1]{#1}
\csname url@samestyle\endcsname
\providecommand{\newblock}{\relax}
\providecommand{\bibinfo}[2]{#2}
\providecommand{\BIBentrySTDinterwordspacing}{\spaceskip=0pt\relax}
\providecommand{\BIBentryALTinterwordstretchfactor}{4}
\providecommand{\BIBentryALTinterwordspacing}{\spaceskip=\fontdimen2\font plus
\BIBentryALTinterwordstretchfactor\fontdimen3\font minus \fontdimen4\font\relax}
\providecommand{\BIBforeignlanguage}[2]{{%
\expandafter\ifx\csname l@#1\endcsname\relax
\typeout{** WARNING: IEEEtran.bst: No hyphenation pattern has been}%
\typeout{** loaded for the language `#1'. Using the pattern for}%
\typeout{** the default language instead.}%
\else
\language=\csname l@#1\endcsname
\fi
#2}}
\providecommand{\BIBdecl}{\relax}
\BIBdecl

\bibitem{pp46}
S.~Song, S.~P. Lichtenberg, and J.~Xiao, ``{Sun RGB-D: A RGB-D Scene Understanding Benchmark Suite},'' in \emph{Proceedings of the IEEE Conference on Computer Vision and Pattern Recognition}, 2015, pp. 567--576.

\bibitem{pp59}
Z.~Han, C.~Zhang, H.~Fu, and J.~T. Zhou, ``{Trusted Multi-View Classification},'' in \emph{9th International Conference on Learning Representations, {ICLR} 2021, Virtual Event, Austria, May 3-7, 2021}.\hskip 1em plus 0.5em minus 0.4em\relax OpenReview.net, 2021.

\bibitem{ppp75}
D.~Zhao, Q.~Gao, Y.~Lu, and D.~Sun, ``{Non-aligned multi-view multi-label classification via learning view-specific labels},'' \emph{IEEE Transactions on Multimedia}, vol.~25, pp. 7235--7247, 2022.

\bibitem{zz67}
Y.~Mo, H.~T. Shen, and X.~Zhu, ``{Unsupervised multi-view graph representation learning with dual weight-net},'' \emph{Information Fusion}, p. 102669, 2024.

\bibitem{zz68}
Y.~Mo, Y.~Chen, Y.~Lei, L.~Peng, X.~Shi, C.~Yuan, and X.~Zhu, ``{Multiplex graph representation learning via dual correlation reduction},'' \emph{IEEE Transactions on Knowledge and Data Engineering}, vol.~35, no.~12, pp. 12\,814--12\,827, 2023.

\bibitem{pp5}
J.-M. P{\'e}rez-R{\'u}a, V.~Vielzeuf, S.~Pateux, M.~Baccouche, and F.~Jurie, ``{MFAS: Multimodal Fusion Architecture Search},'' in \emph{Proceedings of the IEEE/CVF Conference on Computer Vision and Pattern Recognition}, 2019, pp. 6966--6975.

\bibitem{pp6}
Z.~Han, C.~Zhang, H.~Fu, and J.~T. Zhou, ``{Trusted Multi-View Classification With Dynamic Evidential Fusion},'' \emph{IEEE Transactions on Pattern Analysis and Machine Intelligence}, vol.~45, no.~2, pp. 2551--2566, 2023.

\bibitem{pp4}
I.~Csisz{\'a}r, ``{I-Divergence Geometry of Probability Distributions and Minimization Problems},'' \emph{The Annals of Probability}, pp. 146--158, 1975.

\bibitem{pp70}
F.~Nie, J.~Li, X.~Li \emph{et~al.}, ``{Self-weighted multiview clustering with multiple graphs.}'' in \emph{IJCAI}, 2017, pp. 2564--2570.

\bibitem{pp3}
F.~Nielsen, K.~Sun, and S.~Marchand-Maillet, ``{On H{\"o}lder Projective Divergences},'' \emph{Entropy}, vol.~19, no.~3, p. 122, 2017.

\bibitem{pp72}
A.~Dosovitskiy, L.~Beyer, A.~Kolesnikov, D.~Weissenborn, X.~Zhai, T.~Unterthiner, M.~Dehghani, M.~Minderer, G.~Heigold, S.~Gelly, J.~Uszkoreit, and N.~Houlsby, ``{An Image is Worth 16x16 Words: Transformers for Image Recognition at Scale},'' in \emph{9th International Conference on Learning Representations, {ICLR} 2021, Virtual Event, Austria, May 3-7, 2021}.\hskip 1em plus 0.5em minus 0.4em\relax OpenReview.net, 2021.

\bibitem{pp71}
A.~Gu and T.~Dao, ``{Mamba: Linear-Time Sequence Modeling with Selective State Spaces},'' \emph{CoRR}, vol. abs/2312.00752, 2023.

\bibitem{zz73}
H.~Guo, J.~Li, T.~Dai, Z.~Ouyang, X.~Ren, and S.~Xia, ``{MambaIR: {A} Simple Baseline for Image Restoration with State-Space Model},'' in \emph{Computer Vision - {ECCV} 2024 - 18th European Conference, Milan, Italy, September 29-October 4, 2024, Proceedings, Part {XVIII}}, ser. Lecture Notes in Computer Science, vol. 15076.\hskip 1em plus 0.5em minus 0.4em\relax Springer, 2024, pp. 222--241.

\bibitem{zz61}
W.~Zhang, Z.~Deng, T.~Zhang, K.~Choi, J.~Wang, and S.~Wang, ``{Incomplete Multiple View Fuzzy Inference System With Missing View Imputation and Cooperative Learning},'' \emph{IEEE Transactions on Fuzzy Systems}, vol.~30, no.~8, pp. 3038--3051, 2022.

\bibitem{zz62}
Z.~Deng, L.~Liang, H.~Yang, W.~Zhang, Q.~Lou, K.~Choi, T.~Zhang, J.~Zhou, and S.~Wang, ``Enhanced multiview fuzzy clustering using double visible-hidden view cooperation and network {LASSO} constraint,'' \emph{IEEE Transactions on Fuzzy Systems}, vol.~30, no.~11, pp. 4965--4979, 2022.

\bibitem{zz63}
W.~Zhang, Z.~Deng, K.~Choi, and S.~Wang, ``End-to-end incomplete multiview fuzzy clustering with adaptive missing view imputation and cooperative learning,'' \emph{IEEE Transactions on Fuzzy Systems}, vol.~31, no.~5, pp. 1445--1459, 2023.

\bibitem{pp10}
R.~Sanghavi and Y.~Verma, ``{Multi-View Multi-Label Canonical Correlation Analysis for Cross-Modal Matching and Retrieval},'' in \emph{Proceedings of the IEEE/CVF Conference on Computer Vision and Pattern Recognition}, 2022, pp. 4701--4710.

\bibitem{pp14}
Z.~Wu, X.~Lin, Z.~Lin, Z.~Chen, Y.~Bai, and S.~Wang, ``Interpretable graph convolutional network for multi-view semi-supervised learning,'' \emph{IEEE Transactions on Multimedia}, vol.~25, pp. 8593--8606, 2023.

\bibitem{ppp76}
S.~Wu, Y.~Zheng, Y.~Ren, J.~He, X.~Pu, S.~Huang, Z.~Hao, and L.~He, ``{Self-Weighted Contrastive Fusion for Deep Multi-View Clustering},'' \emph{IEEE Transactions on Multimedia}, vol.~26, pp. 9150--9162, 2024.

\bibitem{ppp77}
J.~Tan, Y.~Shi, Z.~Yang, C.~Wen, and L.~Lin, ``{Unsupervised multi-view clustering by squeezing hybrid knowledge from cross view and each view},'' \emph{IEEE Transactions on Multimedia}, vol.~23, pp. 2943--2956, 2020.

\bibitem{ppp78}
J.~Gou, N.~Xie, Y.~Yuan, L.~Du, W.~Ou, and Z.~Yi, ``{Reconstructed graph constrained auto-encoders for multi-view representation learning},'' \emph{IEEE Transactions on Multimedia}, vol.~26, pp. 1319--1332, 2023.

\bibitem{pp7}
G.~Shafer, \emph{{A Mathematical Theory of Evidence}}.\hskip 1em plus 0.5em minus 0.4em\relax Princeton University Press, 1976, vol.~42.

\bibitem{zz64}
Q.~Yang, G.~Han, W.~Gao, Z.~Yang, S.~Zhu, and Y.~Deng, ``A robust learning membership scaling fuzzy c-means algorithm based on new belief peak,'' \emph{IEEE Transactions on Fuzzy Systems}, vol.~31, no.~12, pp. 4486--4500, 2023.

\bibitem{zz65}
Y.~Liu, N.~R. Pal, A.~R. Marathe, and C.~Lin, ``Weighted fuzzy dempster-shafer framework for multimodal information integration,'' \emph{IEEE Transactions on Fuzzy Systems}, vol.~26, no.~1, pp. 338--352, 2018.

\bibitem{pp32}
Q.~Li, C.~Zhang, Q.~Hu, H.~Fu, and P.~Zhu, ``{Confidence-Aware Fusion Using Dempster-Shafer Theory for Multispectral Pedestrian Detection},'' \emph{IEEE Transactions on Multimedia}, vol.~25, pp. 3420--3431, 2023.

\bibitem{ppp79}
Z.~Zhang, H.~Wang, J.~Geng, X.~Deng, and W.~Jiang, ``{A New Data Augmentation Method Based on Mixup and Dempster-Shafer Theory},'' \emph{IEEE Transactions on Multimedia}, vol.~26, pp. 4998--5013, 2024.

\bibitem{ppp80}
Q.~Li, C.~Zhang, Q.~Hu, H.~Fu, and P.~Zhu, ``Confidence-aware fusion using dempster-shafer theory for multispectral pedestrian detection,'' \emph{IEEE Transactions on Multimedia}, vol.~25, pp. 3420--3431, 2023.

\bibitem{yue2021rnn}
Y.~Yue, M.~Li, V.~Saligrama, and Z.~Zhang, ``Rnn training along locally optimal trajectories via frank-wolfe algorithm,'' in \emph{2020 25th International Conference on Pattern Recognition (ICPR)}.\hskip 1em plus 0.5em minus 0.4em\relax IEEE, 2021, pp. 10\,532--10\,539.

\bibitem{tian2023fakepoi}
L.~Tian, H.~Yao, and M.~Li, ``Fakepoi: A large-scale fake person of interest video detection benchmark and a strong baseline,'' \emph{IEEE Transactions on Circuits and Systems for Video Technology}, vol.~33, no.~11, pp. 6819--6831, 2023.

\bibitem{li2023dr}
M.~Li, H.~Fu, S.~He, H.~Fan, J.~Liu, J.~Keppo, and M.~Z. Shou, ``Dr-fer: Discriminative and robust representation learning for facial expression recognition,'' \emph{IEEE Transactions on Multimedia}, vol.~26, pp. 6297--6309, 2023.

\bibitem{zhang2024semi}
Y.~Zhang, C.~Li, Z.~Liu, and M.~Li, ``Semi-supervised disease classification based on limited medical image data,'' \emph{IEEE Journal of Biomedical and Health Informatics}, vol.~28, no.~3, pp. 1575--1586, 2024.

\bibitem{lyu2021treernn}
Y.~Lyu, M.~Li, X.~Huang, U.~Guler, P.~Schaumont, and Z.~Zhang, ``Treernn: Topology-preserving deep graph embedding and learning,'' in \emph{2020 25th International Conference on Pattern Recognition (ICPR)}.\hskip 1em plus 0.5em minus 0.4em\relax IEEE, 2021, pp. 7493--7499.

\bibitem{li2024instant3d}
M.~Li, P.~Zhou, J.-W. Liu, J.~Keppo, M.~Lin, S.~Yan, and X.~Xu, ``Instant3d: instant text-to-3d generation,'' \emph{International Journal of Computer Vision}, vol. 132, no.~10, pp. 4456--4472, 2024.

\bibitem{li2021exploiting}
M.~Li, J.~Liu, C.~Zheng, X.~Huang, and Z.~Zhang, ``Exploiting multi-view part-wise correlation via an efficient transformer for vehicle re-identification,'' \emph{IEEE Transactions on Multimedia}, vol.~25, pp. 919--929, 2021.

\bibitem{li2021self}
M.~Li, X.~Huang, and Z.~Zhang, ``Self-supervised geometric features discovery via interpretable attention for vehicle re-identification and beyond,'' in \emph{Proceedings of the IEEE/CVF international conference on computer vision}, 2021, pp. 194--204.

\bibitem{zheng2020lodonet}
C.~Zheng, Y.~Lyu, M.~Li, and Z.~Zhang, ``Lodonet: A deep neural network with 2d keypoint matching for 3d lidar odometry estimation,'' in \emph{Proceedings of the 28th ACM international conference on multimedia}, 2020, pp. 2391--2399.

\bibitem{li2023stprivacy}
M.~Li, X.~Xu, H.~Fan, P.~Zhou, J.~Liu, J.-W. Liu, J.~Li, J.~Keppo, M.~Z. Shou, and S.~Yan, ``Stprivacy: Spatio-temporal privacy-preserving action recognition,'' in \emph{Proceedings of the IEEE/CVF International Conference on Computer Vision}, 2023, pp. 5106--5115.

\bibitem{zz66}
T.~Denoeux, ``Quantifying prediction uncertainty in regression using random fuzzy sets: The ennreg model,'' \emph{IEEE Transactions on Fuzzy Systems}, vol.~31, no.~10, pp. 3690--3699, 2023.

\bibitem{ppp57}
T.-T. Wong, ``{Generalized Dirichlet distribution in Bayesian analysis},'' \emph{Applied Mathematics and Computation}, vol.~97, no. 2-3, pp. 165--181, 1998.

\bibitem{p13}
O.~Barndorff-Nielsen, \emph{{Information and Exponential Families: in Statistical Theory}}.\hskip 1em plus 0.5em minus 0.4em\relax John Wiley \& Sons, 2014.

\bibitem{pp58}
C.~Elkan, ``{Clustering documents with an exponential-family approximation of the Dirichlet compound multinomial distribution},'' in \emph{Proceedings of the 23rd International Conference on Machine Learning}, 2006, pp. 289--296.

\bibitem{pp34}
J.~Aitchison, ``{The Statistical Analysis of Compositional Data},'' \emph{Journal of the Royal Statistical Society: Series B (Methodological)}, vol.~44, no.~2, pp. 139--160, 1982.

\bibitem{pp38}
K.~W. Ng, G.-L. Tian, and M.-L. Tang, ``{Dirichlet and Related Distributions: Theory, Methods and Applications},'' 2011.

\bibitem{pp39}
C.~M. Bishop and N.~M. Nasrabadi, \emph{{Pattern Recognition and Machine Learning}}.\hskip 1em plus 0.5em minus 0.4em\relax Springer, 2006, vol.~4, no.~4.

\bibitem{pp40}
A.~Jsang, \emph{{Subjective Logic: A Formalism for Reasoning Under Uncertainty}}.\hskip 1em plus 0.5em minus 0.4em\relax Springer Publishing Company, Incorporated, 2018.

\bibitem{pp76}
D.~P. Kingma and M.~Welling, ``{Auto-Encoding Variational Bayes},'' in \emph{2nd International Conference on Learning Representations, {ICLR} 2014, Banff, AB, Canada, April 14-16, 2014, Conference Track Proceedings}, Y.~Bengio and Y.~LeCun, Eds., 2014.

\bibitem{pp73}
L.~Tran, M.~Pantic, and M.~P. Deisenroth, ``{Cauchy--Schwarz Regularized Autoencoder},'' \emph{Journal of Machine Learning Research}, vol.~23, no. 115, pp. 1--37, 2022.

\bibitem{pp45}
N.~Silberman, D.~Hoiem, P.~Kohli, and R.~Fergus, ``{Indoor Segmentation and Support Inference from RGBD Images},'' in \emph{Computer Vision--ECCV 2012: 12th European Conference on Computer Vision, Florence, Italy, October 7-13, 2012, Proceedings, Part V 12}.\hskip 1em plus 0.5em minus 0.4em\relax Springer, 2012, pp. 746--760.

\bibitem{pp47}
B.~Zhou, H.~Zhao, X.~Puig, T.~Xiao, S.~Fidler, A.~Barriuso, and A.~Torralba, ``{Semantic Understanding of Scenes Through the ADE20K Dataset},'' \emph{International Journal of Computer Vision}, vol. 127, pp. 302--321, 2019.

\bibitem{pp48}
A.~Dai, A.~X. Chang, M.~Savva, M.~Halber, T.~Funkhouser, and M.~Nie{\ss}ner, ``{ScanNet: Richly-Annotated 3D Reconstructions of Indoor Scenes},'' in \emph{Proceedings of the IEEE Conference on Computer Vision and Pattern Recognition}, 2017, pp. 5828--5839.

\bibitem{zz69}
J.~Yang, R.~Shi, D.~Wei, Z.~Liu, L.~Zhao, B.~Ke, H.~Pfister, and B.~Ni, ``Medmnist v2-a large-scale lightweight benchmark for 2d and 3d biomedical image classification,'' \emph{Scientific Data}, vol.~10, no.~1, p.~41, 2023.

\bibitem{zz70}
J.~Yang, R.~Shi, and B.~Ni, ``Medmnist classification decathlon: A lightweight automl benchmark for medical image analysis,'' in \emph{IEEE 18th International Symposium on Biomedical Imaging (ISBI)}, 2021, pp. 191--195.

\bibitem{pp69}
F.~Nie, J.~Li, X.~Li \emph{et~al.}, ``{Parameter-free auto-weighted multiple graph learning: A framework for multiview clustering and semi-supervised classification.}'' in \emph{IJCAI}, vol.~9, 2016, pp. 1881--1887.

\bibitem{zz60}
S.~Wang, Z.~Chen, S.~Du, and Z.~Lin, ``{Learning deep sparse regularizers with applications to multi-view clustering and semi-supervised classification},'' \emph{IEEE Transactions on Pattern Analysis and Machine Intelligence}, vol.~44, no.~9, pp. 5042--5055, 2021.

\bibitem{pp49}
K.~He, X.~Zhang, S.~Ren, and J.~Sun, ``{Deep Residual Learning for Image Recognition},'' in \emph{Proceedings of the IEEE Conference on Computer Vision and Pattern Recognition}, 2016, pp. 770--778.

\bibitem{pp50}
J.~Deng, W.~Dong, R.~Socher, L.-J. Li, K.~Li, and L.~Fei-Fei, ``{ImageNet: A Large-Scale Hierarchical Image Database},'' in \emph{2009 IEEE Conference on Computer Vision and Pattern Recognition}.\hskip 1em plus 0.5em minus 0.4em\relax IEEE, 2009, pp. 248--255.

\bibitem{pp51}
D.~P. Kingma and J.~Ba, ``{Adam: A Method for Stochastic Optimization},'' in \emph{3rd International Conference on Learning Representations, {ICLR} 2015, San Diego, CA, USA, May 7-9, 2015, Conference Track Proceedings}, 2015.

\bibitem{pp52}
A.~Paszke, S.~Gross, F.~Massa, A.~Lerer, J.~Bradbury, G.~Chanan, T.~Killeen, Z.~Lin, N.~Gimelshein, L.~Antiga \emph{et~al.}, ``{PyTorch: An Imperative Style, High-Performance Deep Learning Library},'' \emph{Advances in Neural Information Processing Systems}, vol.~32, 2019.

\bibitem{pp53}
D.~Du, L.~Wang, H.~Wang, K.~Zhao, and G.~Wu, ``{Translate-to-Recognize Networks for RGB-D Scene Recognition},'' in \emph{Proceedings of the IEEE/CVF Conference on Computer Vision and Pattern Recognition}, 2019, pp. 11\,836--11\,845.

\bibitem{pp54}
X.~Song, S.~Jiang, B.~Wang, C.~Chen, and G.~Chen, ``{Image Representations with Spatial Object-to-Object Relations for RGB-D Scene Recognition},'' \emph{IEEE Transactions on Image Processing}, vol.~29, pp. 525--537, 2019.

\bibitem{pp55}
A.~Ayub and A.~R. Wagner, ``{Centroid Based Concept Learning for RGB-D Indoor Scene Classification},'' in \emph{31st British Machine Vision Conference 2020, {BMVC} 2020, Virtual Event, UK, September 7-10, 2020}.\hskip 1em plus 0.5em minus 0.4em\relax {BMVA} Press, 2020.

\bibitem{pp56}
D.~Du, L.~Wang, Z.~Li, and G.~Wu, ``{Cross-Modal Pyramid Translation for RGB-D Scene Recognition},'' \emph{International Journal of Computer Vision}, vol. 129, no.~8, pp. 2309--2327, 2021.

\bibitem{pp57}
A.~Caglayan, N.~Imamoglu, A.~B. Can, and R.~Nakamura, ``{When CNNs Meet Random RNNs: Towards Multi-Level Analysis for RGB-D Object and Scene Recognition},'' \emph{Computer Vision and Image Understanding}, vol. 217, p. 103373, 2022.

\bibitem{ppp55}
F.~Nie, J.~Li, X.~Li \emph{et~al.}, ``{Self-weighted multiview clustering with multiple graphs.}'' in \emph{IJCAI}, 2017, pp. 2564--2570.

\bibitem{ppp56}
F.~Nie, G.~Cai, and X.~Li, ``{Multi-view clustering and semi-supervised classification with adaptive neighbours},'' in \emph{Proceedings of the AAAI Conference on Artificial Intelligence}, vol.~31, no.~1, 2017.

\bibitem{pppp57}
X.~Wang, Z.~Lei, X.~Guo, C.~Zhang, H.~Shi, and S.~Z. Li, ``{Multi-view subspace clustering with intactness-aware similarity},'' \emph{Pattern Recognition}, vol.~88, pp. 50--63, 2019.

\bibitem{ppp58}
K.~Zhan, F.~Nie, J.~Wang, and Y.~Yang, ``{Multiview consensus graph clustering},'' \emph{IEEE Transactions on Image Processing}, vol.~28, no.~3, pp. 1261--1270, 2018.

\bibitem{ppp59}
Z.~Zhang, L.~Liu, F.~Shen, H.~T. Shen, and L.~Shao, ``{Binary multi-view clustering},'' \emph{IEEE Transactions on Pattern Analysis and Machine Intelligence}, vol.~41, no.~7, pp. 1774--1782, 2018.

\bibitem{ppp60}
S.~Wang, Z.~Chen, S.~Du, and Z.~Lin, ``{Learning Deep Sparse Regularizers With Applications to Multi-View Clustering and Semi-Supervised Classification},'' \emph{IEEE Transactions on Pattern Analysis and Machine Intelligence}, vol.~44, no.~9, pp. 5042--5055, 2022.

\bibitem{ppp61}
X.~Wang, Z.~Lei, X.~Guo, C.~Zhang, H.~Shi, and S.~Z. Li, ``{Multi-view subspace clustering with intactness-aware similarity},'' \emph{Pattern Recognition}, vol.~88, pp. 50--63, 2019.

\bibitem{zz71}
F.~Nielsen and K.~Okamura, ``On the f-divergences between densities of a multivariate location or scale family,'' \emph{Statistics and Computing}, vol.~34, no.~1, p.~60, 2024.

\end{thebibliography}

   \begin{IEEEbiography}[{\includegraphics[width=1in,height=1.25in,clip,keepaspectratio]{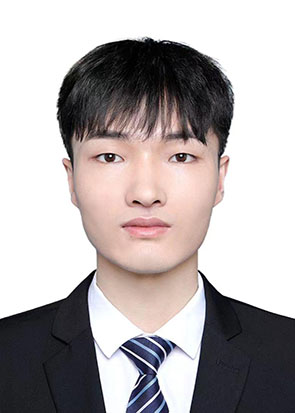}}]{Yan Zhang}
        is currently pursuing a Master's degree at Minzu University of China. He completed his undergraduate studies at the Department of Information and Computing Science of Minzu University of China in 2023, obtaining a Bachelor's degree in the same year. His current research interests revolve around image classification, clustering, and multimodal information fusion.
    \end{IEEEbiography}

    \begin{IEEEbiography}
[{\includegraphics[width=1in,height=1.25in,clip,keepaspectratio]{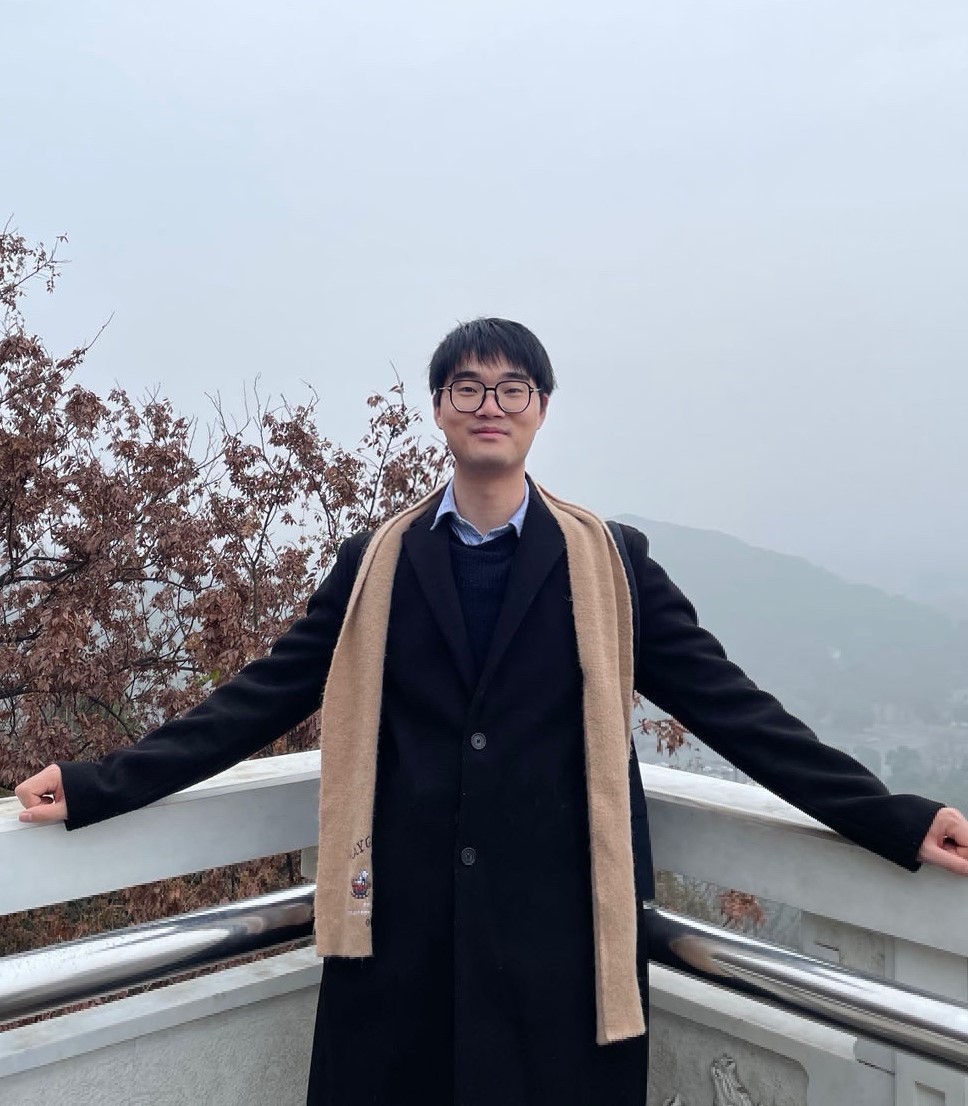}}]{Ming Li} received his Ph.D. degree from National University of Singapore in 2024 and M.S. degree from Peking University in 2018. He is currently a Rising Star Professor and Principle Investigator in Guangdong Laboratory of Artificial Intelligence and Digital Economy (SZ). His research interests lie in AIGC (text-to-image/video/3D generation) and Multi-modal Large Language Models (MLLMs). He worked as a Research Scholar in The University of North Carolina at Chapel Hill and Worcester Polytechnic Institute from Aug. 2018 to Apr. 2021. He has published 17 papers on famous AI/CV conferences and journals, such as IJCV, CVPR, ICCV and IEEE TMM. He serves as a reviewer for IEEE TPAMI, CVPR, ICCV, NeurIPS, AAAI, IEEE TNNLS, IEEE TIP and Neurocomputing.
\end{IEEEbiography}

    \begin{IEEEbiography}[{\includegraphics[width=1in,height=1.25in,clip,keepaspectratio]{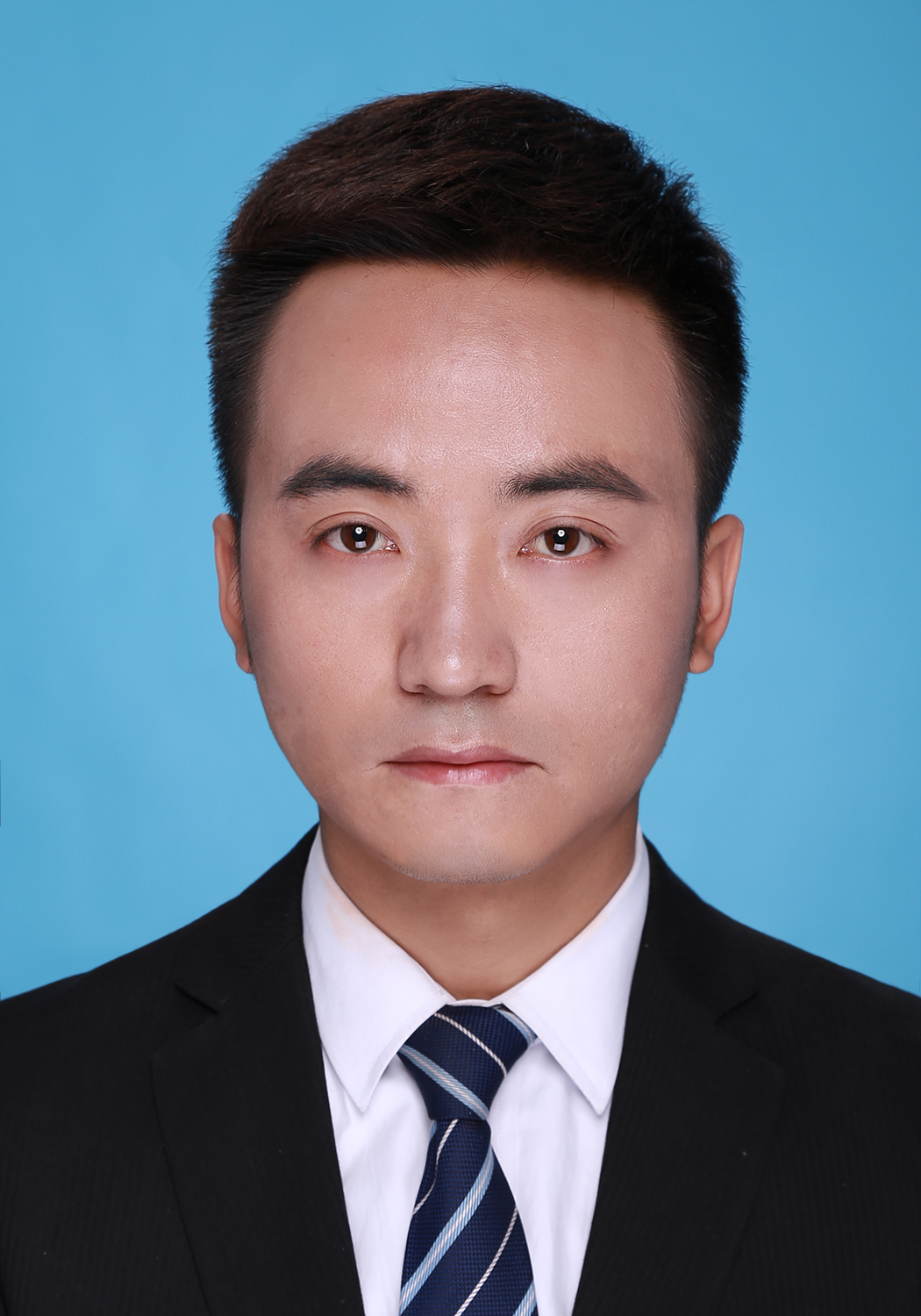}}]{Chun Li}
         earned his M.S. degree in Applied Mathematics from Minzu University of China, Beijing, China, in 2015. He further pursued his academic journey by obtaining a Ph.D. in Computer Application Technology from the University of Chinese Academy of Sciences in Beijing, China. From 2020 to 2022, he served as a postdoctoral fellow at the Harbin Institute of Technology, Shenzhen, China. Presently, he holds the position of Assistant Professor at the Joint Research Center for Computational Mathematics and Control, Shenzhen MSU-BIT University, China. His research endeavors encompass computer vision, generative learning, medical image analysis, and inverse problems.
    \end{IEEEbiography}
    \begin{IEEEbiography}[{\includegraphics[width=1in,height=1.25in,clip,keepaspectratio]{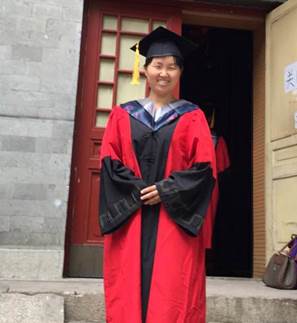}}]{Zhaoxia Liu} a distinguished Professor and Master's Supervisor at the School of Science, Minzu University of China, earned her Ph.D. from the Academy of Mathematics and Systems Science, Chinese Academy of Sciences, in 2003. Her research focuses on partial differential equations, variational methods in image processing, and optimization computation. She has led several key projects funded by the National Natural Science Foundation and other institutions. Dr. Liu has published over 20 papers in reputed journals such as the Journal of Mathematical Analysis and Applications and Applied Mathematics and Computation. She also teaches courses like "Mathematical Analysis."
    \end{IEEEbiography}

    \begin{IEEEbiography}
[{\includegraphics[width=1in,height=1.25in,clip,keepaspectratio]{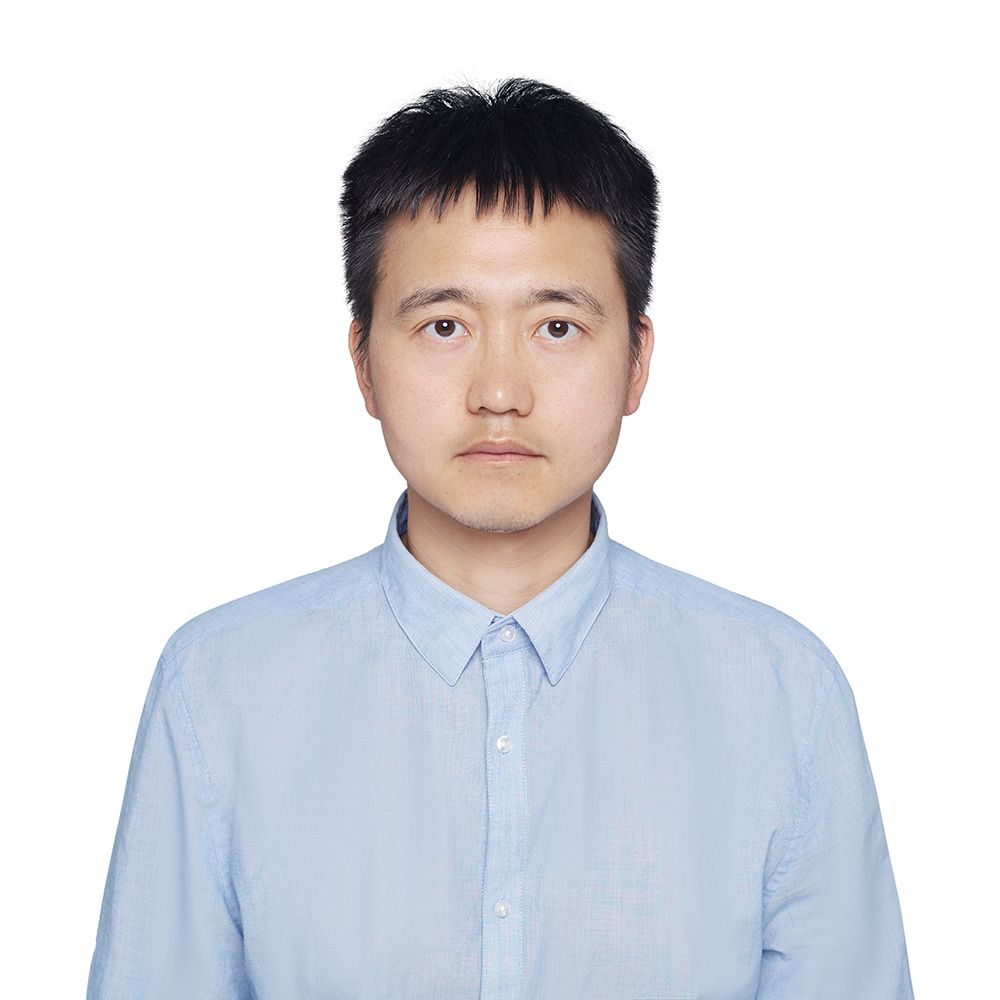}}]{Ye Zhang} received his Ph.D. degree in Mathematical Physics from Lomonosov Moscow State University in 2014. He is currently a Professor and Vice Dean of the Faculty of Computational Mathematics and Cybernetics at Shenzhen MSU-BIT University. His research interests encompass mathematical modeling, inverse problems, and computational analysis. He previously held positions as a Humboldt Fellow at Chemnitz University of Technology, Germany, and served as a researcher at Örebro University and Karlstad University in Sweden. Currently, He serves on the editorial boards of the Journal of Inverse and Ill-posed Problems and Communications on Analysis and Computation.
\end{IEEEbiography}

% \begin{IEEEbiography}
% [{\includegraphics[width=1in,height=1.25in,clip,keepaspectratio]{./Richard_Yu.eps}}]{F. Richard Yu} received the PhD degree in electrical engineering from the University of British Columbia (UBC). His research interests include machine learning and embodied AI, autonomous systems, and blockchain. He has been named in the Clarivate’s list of ``Highly Cited Researchers" in computer science since 2019, Standford’s Top 2\% Most Highly Cited Scientist since 2020. He received several Best Paper Awards from some first-tier conferences, Carleton Research Achievement Awards in 2012 and 2021, and the Ontario Early Researcher Award in 2011. He is a Board Member the IEEE VTS and the Editor-in-Chief for IEEE VTS Mobile World newsletter. He is a Member of the Academia Europaea (MAE), a Fellow of the IEEE, Canadian Academy of Engineering (CAE), Engineering Institute of Canada (EIC), and IET. He is a Distinguished Lecturer of the IEEE.
% \end{IEEEbiography}
    \vfill
\end{document}